\documentclass{article}

\usepackage{microtype}
\usepackage{graphicx}
\usepackage{booktabs}
\usepackage{hyperref}

\usepackage[accepted]{icml2021}

\usepackage{amsmath}
\allowdisplaybreaks
\usepackage{amssymb}
\usepackage{amsthm}
\usepackage{mathtools}
\usepackage{subcaption}
\usepackage{multirow}
\newcommand{\D}{\mathcal{D}}
\newcommand{\E}{\mathop{\mathbb{E}}}
\newcommand{\Var}{\mathop{\mathrm{Var}}}
\newcommand{\Cov}{\mathop{\mathrm{Cov}}}
\newcommand{\Vhat}{\hat{V}}
\newcommand{\thetahat}{\hat{\theta}}
\newcommand{\MC}{\mathrm{MC}}
\newcommand{\IS}{\mathrm{IS}}
\newcommand{\WIS}{\mathrm{WIS}}
\newcommand{\OSIRIS}{\mathrm{OSIRIS}}
\newcommand{\given}{\,|\,}
\newtheorem{theorem}{Theorem}
\newtheorem{lemma}{Lemma}
\newtheorem{proposition}{Proposition}
\newtheorem{definition}{Definition}
\usepackage{enumitem}
\usepackage{array}

\usepackage{multibib}
\newcites{main}{References}
\newcites{appendix}{Supplementary References}

\begin{document}

\twocolumn[
\icmltitlerunning{State Relevance for Off-Policy Evaluation}
\icmltitle{State Relevance for Off-Policy Evaluation}

\begin{icmlauthorlist}
\icmlauthor{Simon P. Shen}{h}
\icmlauthor{Yecheng Jason Ma}{p}
\icmlauthor{Omer Gottesman}{b}
\icmlauthor{Finale Doshi-Velez}{h}
\end{icmlauthorlist}

\icmlaffiliation{h}{Harvard University, Cambridge, MA}
\icmlaffiliation{p}{University of Pennsylvania, Philadelphia, PA}
\icmlaffiliation{b}{Brown University, Providence, RI}

\icmlcorrespondingauthor{Simon P. Shen}{simonshen@fas.harvard.edu}

\icmlkeywords{off-policy evaluation, importance sampling, curse of the horizon}

\vskip 0.3in
]

\printAffiliationsAndNotice{} 

\begin{abstract}
Importance sampling-based estimators for off-policy evaluation (OPE) are valued for their simplicity, unbiasedness, and reliance on relatively few assumptions. However, the variance of these estimators is often high, especially when trajectories are of different lengths.  In this work, we introduce \textit{Omitting-States-Irrelevant-to-Return Importance Sampling} (OSIRIS), an estimator which reduces variance by strategically omitting likelihood ratios associated with certain states.  We formalize the conditions under which OSIRIS is unbiased and has lower variance than ordinary importance sampling, and we demonstrate these properties empirically.
\end{abstract}

\section{Introduction}
\label{sec:intro}
In the context of reinforcement learning, our work focuses on the off-policy evaluation (OPE) problem, where the goal is to estimate the value of a given policy using historical data collected under a different policy \citepmain{sutton2018reinforcement}. OPE is often a necessary step in many real-world applications of reinforcement learning whenever running evaluation policies is costly or risky, for example, in healthcare and education settings \citepmain{Murphy2001,Mandel2014}. In particular, we focus on the OPE approaches based on importance sampling (IS), which correct the historical data to account for the difference between the policies \citepmain{precup2000eligibility}. IS-based estimators are popular for their appealing statistical properties \citepmain{thomas2016data,jiang2016doubly,farajtabar2018more,thomas2015safe}.

However, current IS-based estimators struggle in scenarios involving trajectories of different lengths.  In these settings, IS-based estimators have large variance that increases with trajectory length because IS weights are products of likelihood ratios \citepmain{doroudi2018importance}. Resolving this issue is important as many domains have trajectories with different lengths: in health settings, patients' length of stays may vary drastically; in education settings, students may spend different amounts of time using various online tools.

Motivated by the observation that IS variance is driven by a large and varying number of likelihood ratios, we present a new estimator, \textit{\textbf{O}mitting-\textbf{S}tates-\textbf{I}rrelevant-to-\textbf{R}eturn \textbf{I}mportance \textbf{S}ampling} (OSIRIS), which strategically omits likelihood ratios associated with certain states. The goal of the estimator is to reduce IS variance while introducing minimal bias. We first identify the variance contributed to ordinary IS by likelihood ratios that would be omitted by OSIRIS. This analysis motivates the method's idea to omit likelihood ratios corresponding to ``irrelevant'' states, where the action taken does not affect the trajectory return, and this omission criterion keeps OSIRIS unbiased. Based on this criterion, we describe a practical algorithm using a statistical test to estimate state relevance. Finally, we experimentally validate this implementation of OSIRIS on a suite of discrete- and continuous-state environments. Because the estimator's procedure is to set ``irrelevant'' likelihood ratios to 1, it can be easily used alongside other variants of importance sampling estimators.

\section{Background}
\label{sec:background}

\paragraph{Markov Decision Process}

We consider a standard reinforcement learning framework in which an agent, characterized by a policy $\pi$, interacts with a finite Markov decision process (MDP), characterized by a tuple $(\mathcal{S},\mathcal{A},P,R,\gamma)$. $\mathcal{S}$ and $\mathcal{A}$ represent the state and action spaces, respectively; $P(s'\given s,a)$ and $R(s,a)$ represent the transition distribution and the reward function, respectively; and $\gamma \in [0,1]$ is the temporal discount factor. The agent starts at an initial state $s_1$ drawn from the initial state distribution $P(s_1)$. At each time step $t$, the agent performs an action $a_t \sim \pi(\cdot\given s_t)$, observes reward $r_t = R(s_t,a_t)$, and transitions to state $s_{t+1}\sim P(\cdot \given s_t, a_t)$. Once the agent reaches a terminal state (e.g. at time $T+1$), the trajectory is complete and is defined as $\tau = (s_1, a_1, r_1,\ldots, s_T, a_T, r_T, s_{T+1})$. The discounted rewards collected between any times $t_1$ and $t_2$ in trajectory $\tau$ is defined as $g_{t_1:t_2}(\tau) \equiv \sum_{t=t_1}^{t_2}\gamma^{t-t_1} r_{t}$, and the full trajectory return is simply denoted by $g(\tau) \equiv g_{1:T}(\tau)$.

\paragraph{Policy Evaluation}

In the policy evaluation problem, we are given historical data as a batch of trajectories $\D \equiv \{\tau^{(1)},...,\tau^{(N)}\}$ collected under a behavior policy $\pi_b$, and we want to estimate the value $V^{\pi_e}\equiv \displaystyle\E_{\tau \sim \pi_e}\left[ g(\tau) \right]$ of an evaluation policy $\pi_e$. We use the notation $\tau\sim\pi$ to indicate that trajectories $\tau$ are sampled from the joint probability distribution $P(\tau;\pi)=P(s_1)\prod_{t=1}^{T}\pi(a_t\given s_t)P(s_{t+1}\given s_t, a_t)$, and $\D\sim\pi$ indicates i.i.d. sampling of $N$ such trajectories.

If $\pi_b=\pi_e$, we can perform \textit{on-policy} evaluation with the unbiased Monte Carlo (MC) estimator $\Vhat_{\MC}^{\pi_e}(\D) \equiv \frac{1}{|\D|} \sum_{\tau\in\D} g(\tau)$.

\paragraph{Importance Sampling}

In many real-world scenarios, $\pi_b\neq\pi_e$, so we can only perform \textit{off-policy} evaluation. Among OPE methods, estimators based on importance sampling (IS) are valued for their simplicity, unbiasedness, and reliance on relatively few assumptions \citepmain{precup2000eligibility}. The ordinary IS estimator is given by the weighted average
\begin{equation} \label{eqn:is-estimator}
    \Vhat_{\IS}^{\pi_e}(\D) \equiv \frac{1}{|\D|} \sum_{\tau\in\D} g(\tau) \rho(\tau)
\end{equation}
where the weight $\rho(\tau)$ is the product of likelihood ratios
\begin{equation} \label{eqn:is-weight}
    \rho(\tau) \equiv \prod_{t=1}^{T} \frac{\pi_e(a_t \given s_t)}{\pi_b(a_t \given s_t)}
\end{equation}
We also introduce shorthand notation: $\rho_t(\tau)\equiv \frac{\pi_e(a_t \given s_t)}{\pi_b(a_t \given s_t)}$ and $\rho_{t_1:t_2}(\tau)\equiv \prod_{t=t_1}^{t_2} \rho_{t}(\tau)$. The IS estimator is an unbiased estimator of $V^{\pi_e}$ but typically has large variance, which is the subject of Section~\ref{sec:variance-sources}. To reduce variance, the commonly used weighted IS (WIS) estimator effectively scales the importance weights $\rho(\tau)$ to be between $0$ and $1$: $\Vhat_{\WIS}^{\pi_e}(\D) = \frac{\sum_{\tau\in\D} g(\tau) \rho(\tau)}{\sum_{\tau\in\D} \rho(\tau)}$ \citepmain{precup2000eligibility}. This estimator becomes biased but is consistent.

\section{Related Work}
\label{sec:related-work}

There has recently been significant interest in improving the accuracy of estimation based on importance sampling. Some approaches have included importance weight truncation \citepmain{Ionides2008,Su2019}, confidence bounds \citepmain{thomas2015hcope,thomas2015hcpi,Papini2019,Metelli2020}, and doubly robust estimation \citepmain{jiang2016doubly,thomas2016data,Su2019}. In this work, we focus on the unique challenges presented by the long-horizon setting \citepmain{doroudi2018importance}.

Instead of weighting entire observed trajectories, a recent family of methods showed promising results by calculating weights using estimates of the steady-state visitation distribution \citepmain{liu2018breaking,Xie2019}. This approach was shown to improve asymptotic behavior with respect to the horizon. However, all IS estimators will suffer when trajectories are long, so our fundamentally different approach is to strategically \textit{shorten} the horizon. Furthermore, there is still merit in trying to treat trajectories as a whole rather than breaking them apart into individual transitions, especially if the state space is believed to be partially observable.

\citetmain{doroudi2018importance} propose a per-horizon estimator which first groups trajectories by their lengths, performs WIS on each group, and finally averages these sub-estimates using horizon-corrective weights. However, if each such group is small, the sub-estimates will have large bias because WIS is biased. Furthermore, it is difficult to estimate the distribution of trajectory lengths under $\pi_e$, which is necessary to compute the horizon-corrective weights.

\citetmain{Guo2017} are motivated by options-based policies and leverage temporal abstraction to modify the per-decision importance sampling (PDIS) estimator, which weights the individual rewards from each transition. For settings without access to well-defined options-based policies, the authors suggest dropping all but the $k$-most recent likelihood ratios from each PDIS weight, where $k$ is chosen to minimize an estimate of the estimator's MSE. However, in addition to the difficulty of accurately estimating MSE, the estimator tends to overweight rewards near the end of a trajectory because they are multiplied by fewer likelihood ratios. Our perspective reveals using state relevance to flexibly omit likelihood ratios without favoring any part of the trajectory a priori.

\citetmain{Rowland2020} present a framework for conditional importance sampling, including the return-conditioned importance sampling (RCIS) estimator, which uses IS weights that are conditioned on trajectory return. This approach is designed to remove noise that is unrelated to determining trajectory return. However, RCIS uses a regressor to fit IS weights that are uncorrected and thus still vulnerable to the trajectory length issues discussed in Section~\ref{sec:variance-sources}. Under this conditional IS framework, OSIRIS weights are conditioned on richer relevance information rather than only the trajectory returns, and OSIRIS's likelihood ratio omission is designed to directly address trajectory length issues.

\section{Variance Reduction by Omitting Likelihood Ratios}

\subsection{Sources of Trajectory Length Variance}
\label{sec:variance-sources}

We are motivated by the observation that a major source of variance in IS-based estimators is the \textit{large} and \textit{variable} number of likelihood ratios in the IS weight.

Because the IS weight is the product of likelihood ratios, each of which is a random variable, the IS weight tends to have larger variance when trajectories are \textit{long}:
\begin{proposition} \label{prop:variance-length}
    For any subsets $\mathcal{T}_1 \subsetneq \mathcal{T}_2 \subseteq \{1, \ldots, T\}$, if $\pi_e \neq \pi_b$, and $\rho_1(\tau),\ldots,\rho_T(\tau)$ are mutually independent, then
    \begin{multline}
        \Var_{\tau\sim\pi_b}\big[  
            \prod_{t \in \mathcal{T}_1}
            \rho_t(\tau)
            \given s_1, \ldots, s_T
        \big]
        \\
        <
        \Var_{\tau\sim\pi_b}\big[ 
            \prod_{t \in \mathcal{T}_2}
            \rho_t(\tau)
            \given s_1, \ldots, s_T
        \big]
    \end{multline}
\end{proposition}

The proof is in Appendix~\ref{appendix-sec:pf-variance-length}. To address this issue, we propose to strategically omit likelihood ratios from the IS weight product by setting them to 1, which has zero variance.

Furthermore, when trajectory lengths are \textit{variable}, they contribute variance that is not inherently meaningful to OPE. Intuitively, there are at least two kinds of information that can be represented in the IS weight: the individual actions taken during the trajectory and the trajectory length. First, each likelihood ratio in the IS weight measures how well a transition in the historical data follows the evaluation policy. This feature is intentional so that, in expectation, the overall IS weight corrects for the difference between the behavior and evaluation policies. However the distribution of the IS weight is also directly related to the number of likelihood ratios multiplied together. This is a result of the skew of the distribution of likelihood ratios: because the behavior probability $\pi_b$ appears in the denominator of the likelihood ratio, it is very rare to observe large/exploding likelihood ratios for transitions in the finite historical data, which are sampled from $\pi_b$, but it can be very common to observe small/vanishing likelihood ratios. The effect of this skew is that longer trajectories tend to have smaller IS weights because they multiply more likelihood ratios. This idea is formalized in:
\begin{proposition}
    For any subsets $\mathcal{T}_1\subsetneq\mathcal{T}_2\subseteq\{1, \ldots, T\}$, if $\pi_e \neq \pi_b$, then
    \begin{equation}
        \displaystyle
        \E_{\tau\sim\pi_b}\big[ \log \prod_{t\in\mathcal{T}_1} \rho_t(\tau) \big]
        >
        \displaystyle
        \E_{\tau\sim\pi_b}\big[ \log \prod_{t\in\mathcal{T}_2} \rho_t(\tau) \big]
    \end{equation}
\end{proposition}

The proof is in Appendix~\ref{appendix-sec:pf-jensen}. The log allows us to easily reveal the relationship between trajectory length and IS weight, and it is appropriate because the IS weight multiplies likelihood ratios together. Although the IS estimator is unbiased in expectation, this relationship can be problematic for finite data sizes and especially when trajectory lengths are long and highly variable. In these cases, the IS weights can become dominated by the information about the number of likelihood ratios rather than the meaningful information about how well the trajectory follows the evaluation policy.

\subsection{Omission of Meaningless Likelihood Ratios}
\label{sec:general-omission}

We have identified two problems: when trajectory lengths are \textit{variable}, IS tends to overweight short trajectories in a way that is not inherently meaningful; and when trajectory lengths are \textit{long}, the extra likelihood ratios contribute extra IS variance overall. Our goal is then to strategically omit likelihood ratios in a way that preserves/emphasizes meaningful variance related to the actions taken in the historical data while minimizing meaningless variance that is only related to trajectory length. We begin by decomposing the IS estimator variance, which will suggest such a method.

Assume we have a mapping $\theta' : \mathcal{S} \to \{0,1\}$ that identifies which likelihood ratios should be kept vs omitted in the IS weight.\footnote{The analysis here in Section~\ref{sec:general-omission} works for any general mapping $\theta'$. But we write the input space as just $\mathcal{S}$ in order to smoothly introduce the idea of \textit{state} relevance in Section~\ref{sec:state-indep}.} Omitting likelihood ratios is equivalent to setting them to 1. This procedure can easily be applied to any IS-based estimator (see Extensions in Section~\ref{sec:state-indep}), but for now we formally define the procedure on the ordinary IS estimator:
\begin{definition} \label{def:general-osiris}
    Given any mapping $\theta' : \mathcal{S} \to \{0,1\}$, the OSIRIS weight is defined as
    \begin{equation} \label{eqn:OSIRIS-weight}
        \rho_{\theta'}(\tau) \equiv \prod_{t=1}^{T} \big[ \rho_t(\tau) \big]^{\theta'(s_t)}
    \end{equation}
    and accordingly, the OSIRIS estimator is defined as
    \begin{equation} \label{eqn:OSIRIS-estimator}
        \Vhat_{\OSIRIS}^{\pi_e}(\D; \theta') \equiv \frac{1}{|\D|}\sum_{\tau\in\D}
        g(\tau)
        \,
        \rho_{\theta'}(\tau)
    \end{equation}
\end{definition}

We also introduce notation for the product of the omitted likelihood ratios: $\rho_{\theta'}^\complement(\tau) \equiv \prod_{t=1}^{T} \big[ \rho_t(\tau) \big]^{1 - \theta'(s_t)}$. This quantity directly relates the OSIRIS estimator to the ordinary IS estimator:
\begin{equation}
    \Vhat_\IS^{\pi_e}(\tau) = \Vhat_\OSIRIS^{\pi_e}(\tau; \theta')\cdot \rho_{\theta'}^\complement(\tau)
\end{equation}
where we have abused notation by writing $\Vhat^{\pi_e}(\tau)$ to represent the single-trajectory estimator $\Vhat^{\pi_e}(\{\tau\})$. We use this fact to decompose the IS estimator variance:
\begin{theorem} \label{thm:OSIRIS-variance}
    Given any mapping $\theta' : \mathcal{S} \to \{0,1\}$, the variance of the OSIRIS estimator is:
    \begin{subequations}
        \begin{gather}
            \Var_{\D\sim\pi_b}[\Vhat_{\OSIRIS}^{\pi_e}(\D; \theta')]
        	=
            \Var_{\D\sim\pi_b}[\Vhat_{\IS}^{\pi_e}(\D)]
            \hspace{1in}\nonumber
        	\\
            + \frac{1}{|\D|}
            \E_{\D\sim\pi_b}[\Vhat_{\IS}^{\pi_e}(\tau^{(1)})]^2
        	- \frac{1}{|\D|}
        	\E_{\D\sim\pi_b}[\Vhat_{\OSIRIS}^{\pi_e}(\tau^{(1)}; \theta')]^2
        	\label{eqn:OSIRIS-variance-centers}
        	\\
        	- \frac{1}{|\D|}
            	\E_{\D\sim\pi_b}[
            		\Vhat_{\OSIRIS}^{\pi_e}(\tau^{(1)}; \theta')^2
            		]
            	\,
            	\Var_{\D\sim\pi_b}[
            	    \rho_{\theta'}^\complement(\tau^{(1)})
            	    ]
            \label{eqn:OSIRIS-variance-var}
        	\\
        	\hspace{0.25in}
        	- \frac{1}{|\D|}
        	    \Cov_{\D\sim\pi_b}[ \Vhat_{\OSIRIS}^{\pi_e}(\tau^{(1)}; \theta')^2
        	    ,\, \rho_{\theta'}^\complement(\tau^{(1)})^2 ]
        	\label{eqn:OSIRIS-variance-cov}
        \end{gather}
        \label{eqn:OSIRIS-variance}
    \end{subequations}
\end{theorem}

The derivation is in Appendix~\ref{appendix-sec:OSIRIS-variance-pf}. Term (\ref{eqn:OSIRIS-variance-centers}) adjusts for the locations of the respective estimator distributions. It is generally much smaller in magnitude than all other terms, which is plausible by the Cauchy-Schwarz inequality.\footnote{Terms~(\ref{eqn:OSIRIS-variance-var}) and (\ref{eqn:OSIRIS-variance-cov}) are of the form $\E[x^2]$ while the terms in (\ref{eqn:OSIRIS-variance-centers}) are of the form $\E[x]^2$.} Term (\ref{eqn:OSIRIS-variance-var}) represents the variance of the omitted likelihood ratios. This term cannot be positive, so it can only act to decrease variance. The key conclusion is about Term~(\ref{eqn:OSIRIS-variance-cov}), which mirrors\footnote{By definition, the variance involves \textit{second} moments, so the variables in Term~(\ref{eqn:OSIRIS-variance-cov}) are squared but not in Equation~\ref{eqn:OSIRIS-mean}.} the bias term:

\begin{theorem} \label{thm:OSIRIS-bias}
    Given any mapping $\theta' : \mathcal{S} \to \{0,1\}$, the mean of the OSIRIS estimator is:
    \begin{multline} \label{eqn:OSIRIS-mean}
        \E_{\D\sim\pi_b}[\Vhat_\OSIRIS^{\pi_e}(\D; \theta')] = \\
            V^{\pi_e}
            -
            \Cov_{\D\sim\pi_b}[ \Vhat_\OSIRIS^{\pi_e}(\tau^{(1)}; \theta'), \, \rho_{\theta'}^\complement(\tau^{(1)})]
    \end{multline}
\end{theorem}

The derivation is in Appendix~\ref{appendix-sec:OSIRIS-mean-pf}. If the covariance between $\Vhat_\OSIRIS^{\pi_e}(\tau;\theta')$ and $\rho_{\theta'}^\complement(\tau)$ is zero, then the estimator is unbiased by Theorem~\ref{thm:OSIRIS-bias} and likely has reduced variance by Theorem~\ref{thm:OSIRIS-variance}. This observation suggests an algorithm that chooses $\theta'$ to minimize this covariance term.

However, accurate estimation of the covariance term is challenging: high-variance estimates of the covariance can introduce large variance to the OPE estimator by outputting drastically different $\theta'$ per sample, which would increase length variance. Accurate estimation of the covariance term is further complicated by independence requirements that restrict the usable data. First, Equations~\ref{eqn:OSIRIS-variance-cov} and \ref{eqn:OSIRIS-mean} say the covariance term should be estimated over i.i.d. data sets $\D$, each of which has its own $\theta'$; but in practice, any calculation of $\theta'$ will probably incorporate all data in a single sample of $\D$. Furthermore, Theorems~\ref{thm:OSIRIS-variance} and \ref{thm:OSIRIS-bias} use the fact that $\E_{\D\sim\pi_b}[\rho_{\theta'}^\complement(\tau^{(1)})] = 1$, which is generally only true if $\theta'$ is calculated using $\D \setminus \{\tau^{(1)}\}$ (Appendix~\ref{appendix-sec:OSIRIS-expect-pf}).

Towards getting around these problems associated with picking $\theta'$, we further interpret the covariance term between $\Vhat_\OSIRIS^{\pi_e}(\tau;\theta')$ and $\rho_{\theta'}^\complement(\tau)$ as a measure of the statistical dependence between the trajectory outcome under $\pi_e$ and the product of omitted likelihood ratios, respectively. As discussed in Section~\ref{sec:variance-sources}, the product of likelihood ratios can represent both the number of transitions included and how well those transitions follow $\pi_e$. Assuming that trajectory length is not inherently meaningful to trajectory outcome, the information contained in $\rho_{\theta'}^\complement(\tau)$ is then primarily about the actions taken during the trajectory. Indeed, the covariance term involving the \textit{product} of likelihood ratios is the sum of covariances involving each \textit{individual} likelihood ratio (Appendix~\ref{appendix-sec:cov-indivlrs}). This decomposition suggests omitting individual likelihood ratios that are independent of the kept terms in the IS estimator. We refine this idea towards a practical estimator algorithm in the next section.

\section{State Relevance}
\label{sec:state-indep}

We have motivated the idea to reduce IS variance without introducing bias by omitting individual likelihood ratios where there is no statistical dependence between the action taken and the trajectory outcome. Now we present another perspective of this idea, suggesting a practical algorithm that circumvents the challenges associated with directly estimating the covariance term.

We want to calculate $\theta$ to measure the dependence of the trajectory outcome on each individual action taken. Because the actions are sampled from policies that are conditioned on states, we assume that there are some states where the action taken does not matter to the trajectory outcome. We formalize this idea as the relevance of a state:
\begin{definition} \label{def:state-relevance}
    A state $s\in\mathcal{S}$ is irrelevant if
    \begin{equation} \label{eqn:irrelevance-transition}
        \E_{\tau\sim\pi_e}[g_{t:T}(\tau)\given s_{t}=s,\, a_{t}=a] = \mathrm{constant}
        , \quad
        \forall a \in \mathcal{A}.
    \end{equation}
    Otherwise, $s$ is relevant. Using this condition, we define the true relevance mapping $\theta : \mathcal{S} \to \{0,1\}$ where $\theta(s)=0$ if $s$ is irrelevant, and $\theta(s)=1$ if $s$ is relevant.\footnote{If Equation~\ref{eqn:irrelevance-transition} is true for some $t$, then it must be true for all $t$ by the MDP setup.}
\end{definition}

In other words, a state is irrelevant if the average return-to-go is not affected by the action taken in that state. For example, if state $s^A$ always takes the agent to $s^B$ and no reward is given for that transition, then $s^A$ is considered irrelevant. Or if paths diverge out of $s^A$ but later converge to $s^B$ and no reward is given for those transitions, then $s^A$ is still considered irrelevant. $s^A$ can even be irrelevant if non-zero rewards are given, as long as the expected value of the return-to-go (aka state-action value function $Q^{\pi_e}(s,a)\equiv\E_{\tau\sim\pi_e}[g_{t:T}(\tau)\given s_{t}=s,\, a_{t}=a]$) remains the same regardless of the action taken in $s^A$. Likelihood ratios corresponding to all transitions between $s^A$ and $s^B$ still get multiplied in the ordinary IS weight. But these likelihood ratios have no effect in correcting the difference between $\pi_e$ and $\pi_b$, so they just add meaningless variance to the IS estimator.

Definition~\ref{def:state-relevance} tells us that in these irrelevant states, the trajectory outcome is unaffected by the action taken. As such, we can pretend that in irrelevant states, the evaluation policy will actually draw actions from $\pi_b$  rather than $\pi_e$. This idea is formalized in:

\begin{lemma} \label{lem:composite-policy}
    Let $\pi_e'$ be a composite policy:
    \begin{equation}
        \pi_e'(a\given s; \theta) \equiv \begin{cases}
            \pi_e(a\given s) & \text{if } \theta(s)=1 \text{ (relevant)} \\
            \pi_b(a\given s) & \text{if } \theta(s)=0\text{ (irrelevant)}
        \end{cases}
    \end{equation}
    Then the policy values of $\pi_e$ and $\pi_e'$ are equal:
    \begin{equation}
        V^{\pi_e}=V^{\pi_e'}
    \end{equation}
\end{lemma}

The proof is in Appendix~\ref{appendix-sec:pf-composite-policy-lem}. Performing importance sampling while treating $\pi_e'$ as the evaluation policy is equivalent to setting the importance sampling likelihood ratios to $\frac{\pi_b(a\given s)}{\pi_b(a\given s)} = 1$ wherever $s$ is irrelevant. This perspective is exactly the likelihood ratio omission strategy in OSIRIS. Thus, the OSIRIS estimator using the true state relevance mapping is unbiased:
\begin{theorem}
    \label{thm:OSIRIS-bias-ideal}
    Given the true relevance mapping $\theta$, the mean of the OSIRIS estimator is
    \begin{equation}
        \E_{\D\sim\pi_b}[\Vhat_{\OSIRIS}^{\pi_e}(\D; \theta)] = V^{\pi_e}
    \end{equation}
\end{theorem}

The proof is in Appendix~\ref{appendix-sec:unbiased-ideal}, where we also show that the estimator remains unbiased if we keep irrelevant likelihood ratios, but it is biased if we omit relevant likelihood ratios.

\paragraph{Implementation}

A key assumption of our analysis so far is that we have access to the true relevance mapping $\theta$, but practically, we will almost always need to estimate relevance from the historical data. Definition~\ref{def:state-relevance} naturally suggests a class of algorithms to do so by estimating the state-action value function $Q^{\pi_e}(s,a)$ and then comparing the estimate $\hat{Q}^{\pi_e}(s,a)$ for different actions $a$ within state $s$. If $\hat{Q}^{\pi_e}(s,a)$ is (approximately\footnote{Inevitably, $\hat{Q}^{\pi_e}$ will be an imperfect estimator. But this error could be advantageous because it effectively enforces a softer version of Definition~\ref{def:state-relevance}, which may allow the estimator to also ignore states that are ``mostly'' irrelevant.}) constant for all actions $a$, then the algorithm should consider $s$ irrelevant. Otherwise $s$ should be relevant.

We will focus on presenting one possible implementation, summarized in Algorithm~\ref{alg:estimate-relevance}, that uses IS with a statistical test because it is simple yet effective and can guarantee estimator consistency. For each visit to state $s$ in the historical data, we pre-calculate the associated return-to-go $g_{t:T}(\tau^{(n)})$ multiplied by the IS weight-to-go $\rho_{t:T}(\tau^{(n)})$. This is an unbiased estimate of the expected return-to-go under the evaluation policy $\pi_e$ after taking action $a$ in state $s$ (aka $Q^{\pi_e}(s,a)$). Although these samples likely come from different times $t$, the MDP setup says they still come from the same distribution conditioned on $s$. Thus, for a given $s$, we collect these estimates together for all visits to $s$ into either the list $\mathcal{G}_+$ or $\mathcal{G}_-$ depending on whether the likelihood ratio corresponding to the transition is $>1$ or $\leq 1$, respectively. This procedure effectively produces a binary action space, which allows us to use popular two-sample statistical tests\footnote{If Smirnov's non-parametric test \citepmain{Hodges1958} is used instead, we observe the same qualitative trends as our reported results (Appendix~\ref{appendix-sec:alternative-implementations}).} to compare $\mathcal{G}_+$ and $\mathcal{G}_-$. In particular, we use Welch's $t$-test \citepmain{WelchTTest} where the null hypothesis is Equation~\ref{eqn:irrelevance-transition} with $|\mathcal{A}|=2$. Because the statistical test assumes the individual samples in $\mathcal{G}_+$ and $\mathcal{G}_-$ are i.i.d., we need to assume that state $s$ is visited at most once per trajectory, which is otherwise sampled i.i.d.

\begin{algorithm}[tb]
   \caption{Estimating state relevance $\thetahat(s;\D)$}
   \label{alg:estimate-relevance}
\begin{algorithmic}
   \STATE {\bfseries Input:} state $s$, data $\D$, significance level $\alpha$
   \STATE Initialize $\mathcal{G}_+\gets \emptyset$ and $\mathcal{G}_-\gets \emptyset$.
   \FOR{$n=1$ {\bfseries to} $N$ {\bfseries and} $t=1$ {\bfseries to} $T$}
   \IF{$s_t^{(n)} = s$}
   \IF{$\rho_t(\tau^{(n)}) > 1$}
   \STATE Append $\mathcal{G}_+ \gets \mathcal{G}_+ \cup \{g_{t:T}(\tau^{(n)}) \rho_{t:T}(\tau^{(n)})\}$
   \ELSE
   \STATE Append $\mathcal{G}_- \gets \mathcal{G}_- \cup \{g_{t:T}(\tau^{(n)}) \rho_{t:T}(\tau^{(n)})\}$
   \ENDIF
   \ENDIF
   \ENDFOR
   \STATE Perform \citeauthor{WelchTTest}'s two-sample $t$-test comparing the samples $\mathcal{G}_{+}$ and $\mathcal{G}_{-}$ with significance level $\alpha$
   \IF{null hypothesis is rejected}
   \STATE {\bfseries Output:} $\hat{\theta}(s; \D)\equiv 1$ (relevant)
   \ELSE
   \STATE {\bfseries Output:} $\hat{\theta}(s; \D)\equiv 0$ (irrelevant)
   \ENDIF
\end{algorithmic}
\end{algorithm}

By using Equation~\ref{eqn:irrelevance-transition} as the null hypothesis, this state relevance estimation procedure assumes that states are irrelevant until ``proven'' relevant. Now, we will characterize the effect of this property on accuracy.

Incorrectly identifying a truly irrelevant state as relevant occurs with probability $\leq \alpha$, resulting from a type I error of the statistical test. This error will not affect the bias of the OSIRIS estimator because the likelihood ratios corresponding to truly irrelevant states can take any value without introducing bias (Appendix~\ref{appendix-sec:unbiased-ideal}). Although Term~(\ref{eqn:OSIRIS-variance-cov}) of the variance should also be unaffected by this error because it mirrors the bias term, Term~(\ref{eqn:OSIRIS-variance-var}) will likely increase variance because fewer likelihood ratios are omitted (Proposition~\ref{prop:variance-length}).

Meanwhile, the type II error, of incorrectly identifying a truly relevant state as irrelevant, introduces bias (Appendix~\ref{appendix-sec:unbiased-ideal}). But this bias goes to zero as data size increases because the statistical test is consistent:
\begin{theorem} \label{thm:osiris-consistent}
    If $|\mathcal{A}|=2$ and $\alpha > 0$, then as $|\D|\to\infty$
    \begin{equation}
        \E_{\D\sim\pi_b}[\Vhat_\OSIRIS^{\pi_e}(\D; \thetahat(\,\cdot\,; \D))] = V^{\pi_e}
    \end{equation}
\end{theorem}

The proof is in Appendix~\ref{appendix-sec:pf-consistent}. The purpose of the assumption $|\mathcal{A}|=2$ is to account for the fact that our binary classification (for practical implementation of the statistical test) loses information about the actions. Generally, the action space will be larger than 2, but we expect the estimator to still be consistent if the action space is partitioned in a way that that preserves any dependence between actions and return-to-go. We proposed classifying actions based on their likelihood ratios, which is reasonable because following vs deviating from the evaluation policy could cause drastic differences in trajectory outcome. For example, in a policy improvement setting where the policies are $\epsilon$-greedy, following the optimal action would lead to higher rewards than a random action. Furthermore, our partitioning strategy is based on the likelihood ratios, which are the entities that directly appear in the covariance term in Equations~\ref{eqn:OSIRIS-variance} and \ref{eqn:OSIRIS-mean}. In Appendix~\ref{appendix-sec:alternative-implementations}, we explored another partition strategy, which gave the same qualitative trends as our reported results.

Altogether, because $\alpha$ trades off the probabilities of type I and II error, it can also be seen as a parameter that trades off OSIRIS bias and variance by mixing two OPE approaches: a naive average of behavior trajectory returns ($\alpha=0$ i.e. all likelihood ratios are irrelevant) and an unbiased but high-variance IS estimator ($\alpha=1$ i.e. all likelihood ratios are relevant).\footnote{Interestingly, the jump in the estimator's behavior from $\alpha=0$ to $\alpha>0$ is non-continuous because it involves the first introduction of an IS likelihood ratio.} This perspective also points out that for smaller data size, Algorithm~\ref{alg:estimate-relevance} will more aggressively label states as irrelevant. Equivalently, it prioritizes variance reduction for smaller data size, which is reasonable because variance is inversely related to data size (Theorem~\ref{thm:OSIRIS-variance}).

\paragraph{Alternative Implementations}
Clearly, the estimation in Algorithm~\ref{alg:estimate-relevance} will be more accurate if the lists $\mathcal{G}_+,\mathcal{G}_-$ contain more data, so it effectively requires a discrete state space. Nonetheless, the method can easily be extended to settings with continuous state spaces: in Section~\ref{sec:results}, we show the method still works if we use discretized states to perform the statistical test and use the original continuous states for all other calculations in the estimator.

In Appendix~\ref{appendix-sec:alternative-implementations}, we present empirical results for other possible implementations. Notably, we also tried directly estimating $\hat{Q}^{\pi_e}$ with a neural network model, which can certainly handle comparisons within continuous states and across more than two actions. This approach produces the same qualitative trends as our reported results, which demonstrates the extent to which our analysis is robust to specific implementation choices. An important conclusion from these results is that simpler implementations are generally advantageous because they involve fewer degrees of freedom that can vary across datasets and thus introduce less estimator variance. At the same time, large variance in the estimate of $Q^{\pi_e}$ should be limited in its effect on our overall OSIRIS estimator because $\hat{Q}^{\pi_e}$ is only used for comparisons. Furthermore, the binary output of the comparison $\thetahat(s;\D)\in\{0,1\}$ further weakens any statistical dependence between $\hat{\theta}(s;\D)$ and each individual trajectory in the dataset, which could violate the independence assumptions discussed in Section~\ref{sec:general-omission}.

\paragraph{Extensions}

From the perspective presented in Lemma~\ref{lem:composite-policy}, the OSIRIS procedure is equivalent to performing ordinary IS while treating the composite policy $\pi_e'$ as the evaluation policy, and doing so does not introduce any new bias. Thus, it is natural to use variants of OSIRIS corresponding to any variant of IS \citepmain{thomas2016data,thomas2015safe,jiang2016doubly}. For example, analogous to WIS, we provide empirical results in Section~\ref{sec:results} for OSIRWIS:
\begin{equation} \label{eqn:OSIRWIS-def}
    \Vhat_{\mathrm{OSIRWIS}}^{\pi_e}(\D; \theta') \equiv \frac{\sum_{\tau\in\D}
    g(\tau)
    \,
    \rho_{\theta'}(\tau; \theta')}{\sum_{\tau\in\D}\rho_{\theta'}(\tau)}.
\end{equation}
This principle also directly extends to a step-wise IS framework \citepmain{precup2000eligibility,jiang2016doubly}. Alternatively, because step-wise IS is fundamentally estimating the expected \textit{reward} collected at each time step, it presents an interesting opportunity to define and use the relevance of a state to the reward $\Delta t\in\mathbb{Z}$ time steps away (rather than to the overall trajectory outcome):
\begin{theorem} \label{thm:stepwise}
    Let state $s\in\mathcal{S}$ be irrelevant to the reward $\Delta t$-steps away if
    \begin{equation} \label{eqn:state-irrelevance-timet}
        \E_{\tau\sim\pi_e}[r_{t+\Delta t}\given s_{t}=s,a_{t}=a] = \mathrm{constant}
        ,\quad
        \forall a\in\mathcal{A}.
    \end{equation}
    Otherwise, $s$ is relevant to the reward $\Delta t$-steps away. Using this condition, we define $\theta_{\Delta t}:\mathcal{S}\to\{0,1\}$ where $\theta_{\Delta t}(s)=0$ if $s$ is irrelevant to the reward $\Delta t$-steps away, and otherwise $\theta_{\Delta t}(s)=1$. Then
    \begin{multline} \label{eqn:stepwise-OSIRIS-def}
        \hat{V}_{\substack{\textnormal{step-wise} \\ \textnormal{OSIRIS}}}^{\pi_e}(\D; \theta_{\Delta t})
        \equiv \\
        \frac{1}{|\D|}\sum_{\tau\in\D}
        \sum_{t'=1}^{T}
        \Big(
            \gamma^{t'-1}
            r_{t'}
            \prod_{t=1}^{T} \big[\rho_{t}(\tau)\big]^{\theta_{t' - t}(s_{t})}
        \Big)
    \end{multline}
    is an unbiased estimator of $V^{\pi_e}$.
\end{theorem}

The proof is in Appendix~\ref{appendix-sec:stepwise-OSIRIS}, which uses similar techniques as for Lemma~\ref{lem:composite-policy} and Theorem~\ref{thm:OSIRIS-bias-ideal}. Notice that the per-decision importance sampling (PDIS) estimator \citepmain{precup2000eligibility} can be seen as a special case of this strategy if we use the relevance mapping $\theta_{\Delta t}^{\mathrm{PDIS}}(s) \equiv \mathbf{1}\{\Delta t \geq 0\}$, and this observation is reflected in Equation~\ref{eqn:state-irrelevance-timet} because the MDP setup assumes $a_{t}$ depends only on $s_{t}$. 

Using this step-wise OSIRIS framework, it is straightforward to extend OSIRIS to the doubly robust (DR) estimator \citepmain{jiang2016doubly} by replacing $\rho_{t'}(\tau)r_{t'}$ in each term of Equation~\ref{eqn:stepwise-OSIRIS-def} with the corresponding DR estimate $\rho_{t'}(\tau)\big[r_{t'}-\hat{Q}^{\pi_e}(s_{t'},a_{t'})\big]+\sum_{a\in\mathcal{A}}\pi_e(a\given s_{t'})\hat{Q}^{\pi_e}(s_{t'},a)$ where $\hat{Q}^{\pi_e}$ is some estimate of $Q^{\pi_e}$. Alternatively, the principle behind ordinary OSIRIS can be applied.

\section{Experiments}
\label{sec:results}

In this section, we experimentally validate the efficacy of the likelihood ratio omission procedure of OSIRIS and the relevance estimation procedure in Algorithm~\ref{alg:estimate-relevance}. We demonstrate that they improve estimator accuracy. We demonstrate that this occurs by reducing meaningless variance associated with trajectory length while strategically using $\thetahat(s;\D)$ to identify meaningful variance to be kept.

\paragraph{Environment Descriptions} Detailed descriptions of the environments and methods are in Appendix~\ref{appendix-sec:env-desc}. In summary: All policies are $\epsilon$-greedy where $\epsilon$ is smaller in the evaluation policy. We consider variants of the discrete gridworld shown in Figure~\ref{fig:gridworld-env}. The agent spends a variable number of transitions dilly-dallying in the ``corridor'' before navigating to terminal states that award either $+5$ or $-5$. Compared to this \textit{Dilly-Dallying Gridworld}, the only difference in the \textit{Express Gridworld} variant is that the behavior policy uses an even smaller $\epsilon$ but only in the corridor, which reduces the spread of the number of dilly-dallying transitions and thus alleviates trajectory length issues in IS. We also demonstrate that our implementation in Algorithm~\ref{alg:estimate-relevance} extends to continuous state spaces, specifically the popular benchmark environments \textit{Cart Pole} and \textit{Lunar Lander}. For only the calculation of state relevance, we discretized the state space by creating linearly spaced bins per state dimension. In Cart Pole, the agent receives $+1$ reward for each transition while it can keep an inverted pendulum upright. In Lunar Lander, the agent is rewarded for landing on target, penalized for firing its engines, and harshly penalized for crashing. Unless specified otherwise, results are aggregated from 200 trials where $|\D|=25$ for the Gridworlds and $|\D|=50$ for Cart Pole and Lunar Lander. All code and models used to generate these results are publicly accessible at \href{https://github.com/dtak/osiris}{\texttt{github.com/dtak/osiris}}.

\begin{figure}[t]
\vskip 0.2in
\begin{center}
\centerline{\includegraphics{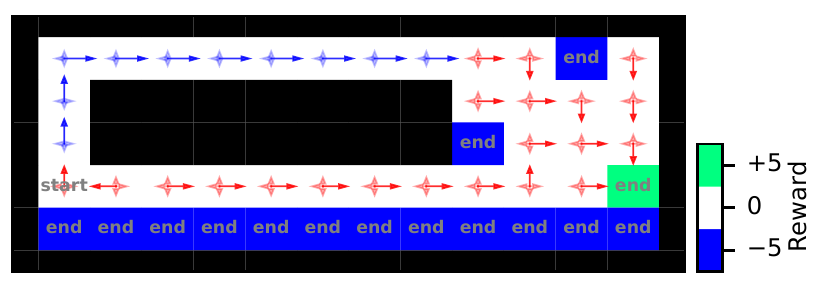}}
\caption{Environment and policies for Gridworld experiments. States with blue arrows comprise the ``corridor.''}
\label{fig:gridworld-env}
\end{center}
\vskip -0.1in
\end{figure}

\paragraph{OSIRIS/OSIRWIS mean-squared errors are generally lower than their IS counterparts.}

\begin{table*}[t]
\caption{Comparison of mean squared errors for IS-based estimators. OSIRIS/OSIRWIS with $\alpha=0.05$ generally outperforms its IS counterparts except in Express Gridworld (see text).} \label{tab:estimator-mse}
\vskip 0.15in
\begin{center}
\begin{small}
\begin{sc}
\begin{tabular}{ll|rrrrr|rr|r}
\toprule
&& IS & WIS & PHWIS & INCRIS & MIS & OSIRIS & OSIRWIS & On-Policy \\
\midrule
\multirow{3}{1in}{\textbf{Dilly-Dallying Gridworld}}
    & Mean & $\mathbf{3.5}$ & $1.1$ & $-0.3$ & $-0.1$ & $5.8$ & $1.3$ & $1.1$ & $4.3$ \\
    & Std  & $6.8$ & $3.5$ & $\mathbf{1.0}$ & $4.8$ & $1.6$ & $2.0$ & $1.8$ & $0.6$ \\
    & RMSE & $6.9$ & $4.7$ & $4.7$ & $6.5$ & $\mathbf{2.2}$ & $3.6$ & $3.7$ & $0.6$ \\
\hline
\multirow{3}{1in}{\textbf{Express Gridworld}}
    & Mean & $3.0$ & $2.7$ & $0.3$ & $\mathbf{4.3}$ & $5.1$ & $0.7$ & $0.8$ & $4.3$ \\
    & Std  & $3.5$ & $2.2$ & $\mathbf{1.1}$ & $4.0$ & $1.7$ & $1.2$ & $1.3$ & $0.6$ \\
    & RMSE & $3.7$ & $2.7$ & $4.1$ & $4.0$ & $\mathbf{1.9}$ & $3.8$ & $3.7$ & $0.6$ \\
\hline
\multirow{3}{1in}{\textbf{Cart Pole}}
    & Mean & $\mathbf{1073.5}$ & $452.2$ & $608.2$ & $1048.4$ & $2498.6$ & $1068.9$ & $759.7$ & $1503.6$ \\
    & Std  & $10202.2$ & $272.5$ & $\mathbf{97.6}$ & $4437.7$ & $583.1$ & $3961.8$ & $318.5$ & $244.8$ \\
    & RMSE & $10211.3$ & $1086.1$ & $900.7$ & $4461.0$ & $1153.3$ & $3985.5$ & $\mathbf{809.2}$ & $244.8$ \\
\hline
\multirow{3}{1in}{\textbf{Lunar Lander}}
    & Mean & $305.6$ & $264.2$ & $239.6$ & $83.8$ & $281.9$ & $\mathbf{244.6}$ & $234.8$ & $245.3$ \\
    & Std  & $768.7$ & $14.9$ & $\mathbf{9.1}$ & $149.6$ & $139.6$ & $55.3$ & $23.5$ & $6.8$ \\
    & RMSE & $771.1$ & $24.1$ & $\mathbf{10.7}$ & $220.2$ & $144.3$ & $55.3$ & $25.7$ & $6.8$ \\
\bottomrule
\end{tabular}
\end{sc}
\end{small}
\end{center}
\end{table*}

The estimator means, standard deviations, and RMSEs for each environment are listed together in Table~\ref{tab:estimator-mse}. OSIRIS/OSIRWIS generally outperforms IS, WIS, PHWIS \citepmain{doroudi2018importance}, INCRIS \citepmain{Guo2017}, and MIS \citepmain{Xie2019}. The RMSE improvement is mostly driven by variance reduction. Express Gridworld is the exception (see discussion below), where OSIRIS/OSIRWIS is not expected to do better because we modified the environment to produce trajectories with less length variability. MIS also performs well in the Gridworlds where the state space is small.

\begin{figure*}[t]
\vskip 0.2in
\begin{center}
 \begin{subfigure}[b]{2.12in}
     \includegraphics{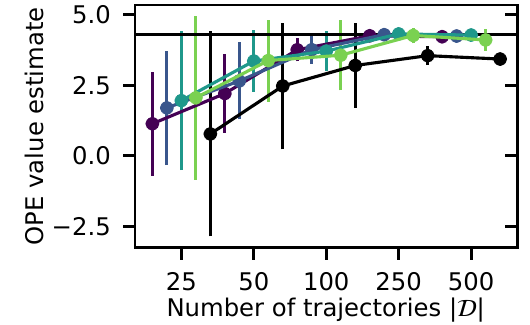}
     \caption{Dilly-Dallying Gridworld}
     \label{fig:consistency-gridworld}
 \end{subfigure}%
 \begin{subfigure}[b]{2.12in}
     \includegraphics{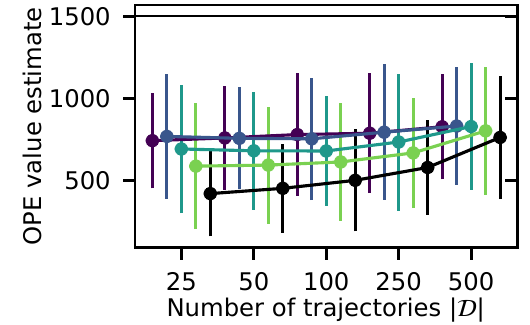}
     \caption{Cart Pole}
     \label{fig:consistency-cartpole}
 \end{subfigure}%
 \begin{subfigure}[b]{2.12in}
     \includegraphics{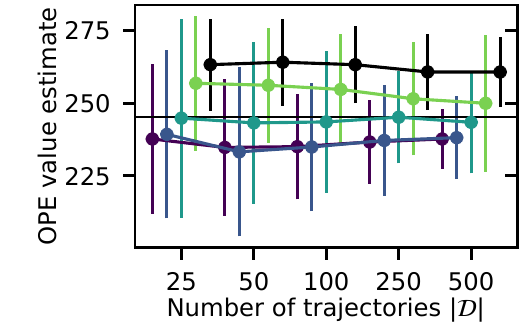}
     \caption{Lunar Lander}
     \label{fig:consistency-lunarlander}
 \end{subfigure}%
 \begin{subfigure}[b]{0.39in}
     \includegraphics{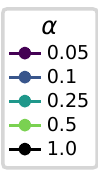}
 \end{subfigure}
\caption{Distributions of OSIRWIS estimates showing estimator consistency. Dots represent means, error bars represent standard deviations, and horizontal line represents the mean of the on-policy MC estimator. Colors indicate $\alpha$ values, where $\alpha=1$ is equivalent to ordinary WIS.}
\label{fig:mseest}
\end{center}
\vskip -0.2in
\end{figure*}

\begin{figure*}[t]
\vskip 0.2in
\begin{center}
 \begin{subfigure}[b]{2.12in}
     \includegraphics{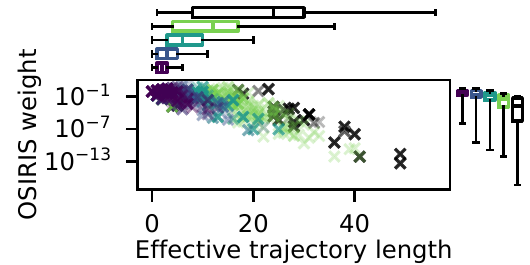}
     \caption{Dilly-Dallying Gridworld}
     \label{fig:scatter-gridworld}
 \end{subfigure}%
 \begin{subfigure}[b]{2.12in}
     \includegraphics{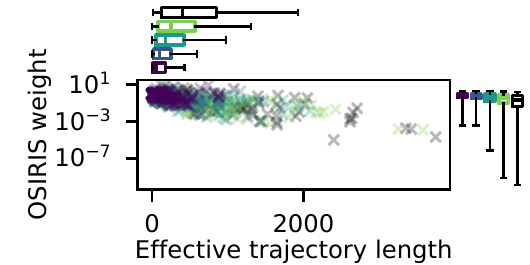}
     \caption{Cart Pole}
     \label{fig:scatter-cartpole}
 \end{subfigure}%
 \begin{subfigure}[b]{2.12in}
     \includegraphics{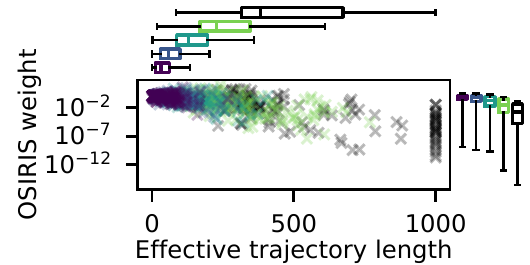}
     \caption{Lunar Lander}
     \label{fig:scatter-lunarlander}
 \end{subfigure}%
 \begin{subfigure}[b]{0.39in}
     \includegraphics{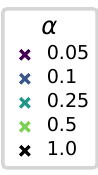}
 \end{subfigure}
\caption{Scatter plots show correlation between OSIRIS weights and effective trajectory lengths $\sum_{t=1}^{T}\thetahat(s_t)$. Boxplots show variance reduction of OSIRIS weights by shortening and evening of the effective trajectory lengths. Colors indicate $\alpha$ values, where $\alpha=1$ is equivalent to ordinary IS.}
\label{fig:scatter}
\end{center}
\vskip -0.2in
\end{figure*}

\paragraph{OSIRWIS bias decreases as $|\D|$ or $\alpha$ increases.} The distributions of OPE value estimates are plotted for different data sizes $|\D|$ and values of $\alpha$ in Figure~\ref{fig:mseest}. These results reflect the expected consistency behavior (Theorem~\ref{thm:osiris-consistent}). As the data size increases, the OSIRWIS estimator mean approaches the true policy value. As $\alpha$ increases, OSIRWIS becomes more similar to the WIS estimator.

\paragraph{Dilly-dallying contributes high variance to IS/WIS but is ignored by OSIRIS/OSIRWIS.} In Section~\ref{sec:variance-sources}, we argued that IS weights correlate with the trajectory length, so if trajectory length varies, then that can directly contribute meaningless variance to the IS estimate. This point is reflected in Figure~\ref{fig:scatter}.

Consider the Dilly-Dallying Gridworld environment, where the agent will dilly-dally in the corridor (blue arrows in Figure~\ref{fig:gridworld-env}). While the random number of dilly-dally transitions affects trajectory length, it does not affect the trajectory return because the agent accumulates zero reward in the corridor and almost always exits the corridor from the same state. IS/WIS is distracted by dilly-dally transitions in the corridor: the likelihood ratios corresponding to these low-probability transitions are less than 1, so each dilly-dally transition makes the IS/WIS weight smaller. A large and variable number of dilly-dally likelihood ratios dominates the IS/WIS weights and masks the informative likelihood ratios corresponding to transitions outside the corridor. As such, the IS/WIS weights become more informative of the number of dilly-dally transitions (trajectory length) than whether a behavior trajectory is representative of the evaluation policy (scatter plot in Figure~\ref{fig:scatter-gridworld}). Because trajectory length is not directly meaningful in this setting, this information hurts IS/WIS accuracy by contributing meaningless variance. Meanwhile, the OSIRIS/OSIRWIS estimator significantly reduces this source of variance (boxplots in Figure~\ref{fig:scatter-gridworld}) by omitting likelihood ratios corresponding to these corridor states (Figure~\ref{fig:gridworld-relevance-dd}).

In contrast, the Express Gridworld variant generates trajectories with less dilly-dallying in the corridor. Although this modification of the environment/policies does not affect policy value, it reduces the meaningless spread of trajectory lengths. Because there are fewer distracting likelihood ratios, the IS/WIS estimator becomes more accurate. Meanwhile, the OSIRIS/OSIRWIS estimator is robust to this modification: estimator accuracy remains constant (Table~\ref{tab:estimator-mse}), and the other reported results still show the same qualitative trends (Appendix~\ref{appendix-sec:gridworldxp}).

While is difficult to precisely interpret dilly-dallying in the high-dimensional Cart Pole and Lunar Lander environments, we still observe the same variance reduction trend (Figures~\ref{fig:scatter-cartpole} and \subref{fig:scatter-lunarlander}).

\paragraph{Estimated state relevance $\thetahat$ identifies key decision points where trajectory outcome is sensitive to the action taken.}

\begin{figure*}[t]
\vskip 0.2in
\begin{center}
\begin{subfigure}[t]{3.25in}
    \includegraphics{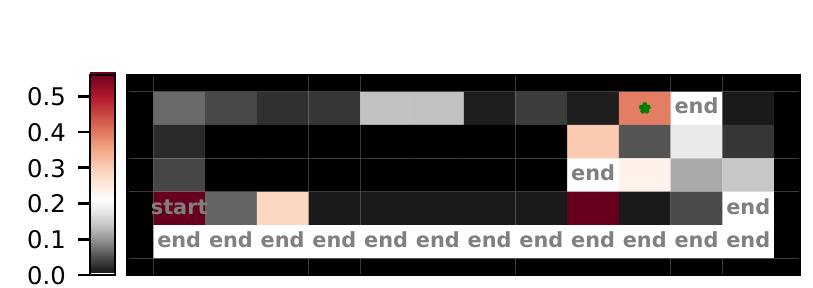}
    \caption{Dilly-Dallying Gridworld}
    \label{fig:gridworld-relevance-dd}
\end{subfigure}%
\begin{subfigure}[t]{1.6875in}
    \includegraphics{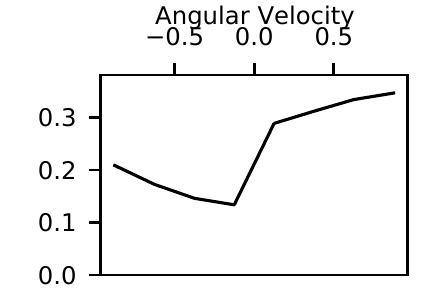}
    \caption{Cart Pole}
    \label{fig:cartpole-relevance}
\end{subfigure}%
\begin{subfigure}[t]{1.6875in}
    \includegraphics{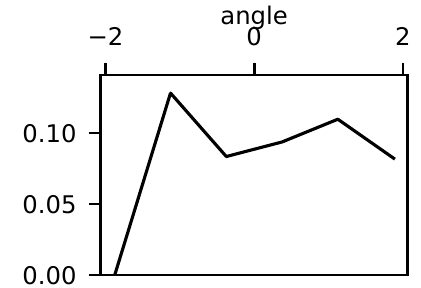}
    \caption{Lunar Lander}
    \label{fig:lunarlander-relevance}
\end{subfigure}
\caption{Mean of estimated state relevance $\thetahat(s)$ from visits to the indicated states is represented by color (\subref{fig:gridworld-relevance-dd}) or on the $y$ axis (\subref{fig:cartpole-relevance}, \subref{fig:lunarlander-relevance}). States identified as relevant (i.e. $\thetahat(s)=1$) are key decision points where trajectory outcome is sensitive to action taken.}
\label{fig:relevance}
\end{center}
\vskip -0.2in
\end{figure*}

Figure~\ref{fig:relevance} plots the proportion of trials in which the indicated state was estimated to be relevant (i.e. $\thetahat(s;\D)=1$). As designed, $\thetahat(s;\D)=1$ identifies key decision points where trajectory outcome is sensitive to the action taken, and often many different trajectory outcomes are accessible from the state.

For example, in the Dilly-Dallying Gridworld, consider the state marked with a green star in Figure~\ref{fig:gridworld-relevance-dd}. Here, if the agent moves east, then the corresponding likelihood ratio $<1$, and the trajectory ends with $-5$ return. If the agent moves south, then the corresponding likelihood ratio $>1$, and the agent will very likely end the trajectory with $+5$ return. Thus, there is a clear relationship between likelihood ratio and trajectory return in this state, which is reflected by OSIRIS's tendency to set $\thetahat(s;\D)=1$. Meanwhile, $\thetahat(s;\D)$ is often $0$ in the corridor (blue arrows in Figure~\ref{fig:gridworld-env}). In any of these corridor states, the likelihood ratio $>1$ if the agent takes the main policy action or $<1$ if it dilly-dallies. However, by the Markov assumption, this decision is independent of the agent's future actions to receive either $+5$ or $-5$ trajectory return. As such, the statistical test should not find any relationship between likelihood ratio and trajectory return, which is reflected by the tendency for $\thetahat(s;\D)= 0$.

The Lunar Lander and Cart Pole environments also demonstrate this principle: $\thetahat(s;D)=1$ identifies relevant states where the agent can avoid crashing by following the optimal policy action but will likely crash if a random action is taken. This establishes a positive correlation between trajectory return and likelihood ratio that is detected by the statistical test. In Cart Pole, the detected relevant states tend to have large angular velocity that can be stabilized by taking the optimal policy action instead of a random action (Figure~\ref{fig:cartpole-relevance}). Similarly, in Lunar Lander, when the agent is not level, it is important whether the agent chooses stabilizing actions (Figure~\ref{fig:lunarlander-relevance}). But interestingly, if it has rotated too far, the state tends to be considered irrelevant -- at this point, there is no hope as the agent will likely crash no matter what action it takes. This principle can be observed in the other state dimensions too (Appendix~\ref{appendix-sec:relevance-other-dims}).

\section{Discussion and Conclusion}

We presented the OSIRIS estimator to reduce importance sampling variance in settings with long and varying trajectory lengths. The algorithm strategically identifies and omits irrelevant likelihood ratios in a way that introduces minimal bias. This procedure can technically be applied to any IS-based estimator or even for the purposes of interpretability \citepmain{gottesman2020interpretable}. 

We also describe when OSIRIS will provide the most benefit. In environments where trajectory length is not directly correlated with the trajectory return, OSIRIS will shine because the covariance term in Equations~\ref{eqn:OSIRIS-variance} and \ref{eqn:OSIRIS-mean} is small: length is meaningless (assumed in these environments) and the individual omitted likelihood ratios are also meaningless (by state irrelevance). In contrast, in environments like Cart Pole, where trajectory length is directly related to trajectory return, likelihood ratio omission could disrupt any naturally occurring relationship between trajectory return and IS weight (aka trajectory length). Depending on the direction of the relationship, which determines the sign of the covariance term (aka the bias term), this could be favorable or harmful by shifting the estimator distribution closer to or further from the true value.  

More broadly, OSIRIS's likelihood ratio omission directly addresses the variance of \textit{long} trajectories, but depending on the environment, it may create more \textit{varying} trajectory length, e.g. by omitting several likelihood ratios in some trajectories while leaving others untouched.  While we did not see this as a dominating factor in our experiments, future extensions could address this issue of uneven lengths by keeping only the $T_{\max}$-most relevant likelihood ratios for some constant $T_{\max}$. Nonetheless, because we omit irrelevant likelihood ratios, we at least expect the OSIRIS weight to give more attention to the relevant likelihood ratios.

\section*{Acknowledgements}

We thank the reviewers for their thoughtful comments. Part of this work was completed while S.P.S. was supported by NSF GRFP. S.P.S. and F.D.V. also acknowledge support from NSF RI-Small 2007076.

\bibliographymain{bibliography}
\bibliographystylemain{icml2021}

\onecolumn
\icmltitlerunning{Supplementary Materials: State Relevance for Off-Policy Evaluation}
\icmltitle{Supplementary Materials: State Relevance for Off-Policy Evaluation}

\appendix
\setcounter{equation}{0}
\setcounter{figure}{0}
\setcounter{table}{0}
\setcounter{lemma}{0}
\setcounter{proposition}{0}
\setcounter{theorem}{0}

\section{Lemma}
\label{appendix-sec:OSIRIS-expect-pf}

Given any mapping $\theta' : \mathcal{S} \to \{0,1\}$,
\begin{equation} \label{lem:lr-expectation}
    \E_{\D\sim\pi_b}[\rho_{\theta'}^\complement(\tau^{(1)})] = 1
\end{equation}

\begin{proof}
    From the definition of the OSIRIS weight,
    \begin{align}
        \E_{\D\sim\pi_b}&[\rho_{\theta'}^\complement(\tau^{(1)})]
        =
        \E_{\D\sim\pi_b}\Big[
            \prod_{t=1}^{T} \Big[
                 \frac{\pi_e(a_t\given s_t)}{\pi_b(a_t\given s_t)}
            \Big]^{1 - \theta'(s_t)}
        \Big]
        \label{appendix-step:lemma-wt}
        \\
        &=
        \int \Big[
            \prod_{t=1}^{T} \Big[
                 \frac{\pi_e(a_t\given s_t)}{\pi_b(a_t\given s_t)}
            \Big]^{1 - \theta'(s_t)} 
            \pi_b(a_t\given s_t)
            P(s_{t+1}\given s_t, a_t)
            P(s_1)
        \Big] \, d\tau
        \label{appendix-step:lemma-jointb}
        \\
        &= \int \Big[ \prod_{t=1}^{T} \pi_e(a_t\given s_t)^{1 - \theta'(s_t)} \pi_b(a_t\given s_t)^{\theta'(s_t)}
        P(s_{t+1}\given s_t, a_t)
        P(s_1)
        \Big] \, d\tau
        \label{appendix-step:lemma-cancel}
        \\
        &=1
        \label{appendix-step:lemma-jointbe}
    \end{align}
    where we use the joint distribution of trajectories generated from $\pi_b$ (\ref{appendix-step:lemma-jointb}) and cancel terms (\ref{appendix-step:lemma-cancel}). Finally, $\theta'(s_t)\in\{0,1\}$ selects either $\pi_e$ or $\pi_b$ as the integrand, both of which integrate to 1 by definition, so the full joint distribution also integrates to 1 (\ref{appendix-step:lemma-jointbe}).
\end{proof}

Notice that if we did not assume $\theta'$ is given, then we should assume that $\theta'$ is calculated using $\D \setminus \tau^{(1)}$. Otherwise, $\theta'$ would have statistical dependence with the transitions in trajectory $\tau^{(1)}$, which could not be resolved in (\ref{appendix-step:lemma-jointb}).

\section{Derivations of Sources of IS Variance}

\subsection{Variance due to Long Trajectory Length}
\label{appendix-sec:pf-variance-length}

\begin{proposition}
    For any subsets $\mathcal{T}_1 \subsetneq \mathcal{T}_2 \subseteq \{1, \ldots, T\}$, if $\pi_e \neq \pi_b$, and $\rho_1(\tau),\ldots,\rho_T(\tau)$ are mutually independent, then
    \begin{equation}
        \Var_{\tau\sim\pi_b}\big[  
            \prod_{t \in \mathcal{T}_1}
            \rho_t(\tau)
            \given s_1, \ldots, s_T
        \big]
        <
        \Var_{\tau\sim\pi_b}\big[ 
            \prod_{t \in \mathcal{T}_2}
            \rho_t(\tau)
            \given s_1, \ldots, s_T
        \big]
    \end{equation}
\end{proposition}

\begin{proof}
    We first consider the variance of the product of general mutually independent random variables $X_1,\ldots,X_N$:
    \begin{align}
        \Var\big[\prod_{n=1}^{N} X_n\big]
        &= \E\big[\prod_{n=1}^{N} X_n^2\big] - \E\big[\prod_{n=1}^{N} X_n\big]^2
        \label{appendix-step:var-def}
        \\
        &= \prod_{n=1}^{N}\E\big[ X_n^2\big] - \prod_{n=1}^{N}\E\big[ X_n\big]^2
        \label{appendix-step:var-indep}
        \\
        &= \prod_{n=1}^{N}\Big(
            \E\big[ X_n^2\big] - \E\big[ X_n\big]^2 + \E\big[ X_n\big]^2
        \Big) - \prod_{n=1}^{N}\E\big[ X_n\big]^2
        \label{appendix-step:var-expand}
        \\
        &= \prod_{n=1}^{N}\Big(
            \Var\big[ X_n\big] + \E\big[ X_n\big]^2
        \Big) - \prod_{n=1}^{N}\E\big[ X_n\big]^2
        \label{appendix-step:var-def-again}
    \end{align}
    where we use the definition of variance (\ref{appendix-step:var-def}), use the assumption that the variables are independent (\ref{appendix-step:var-indep}), introduce the same term (\ref{appendix-step:var-expand}), and reuse the definition of variance (\ref{appendix-step:var-def-again}).
    
    Using this fact, the inequality is equivalent to
    \begin{multline}
        \prod_{t \in \mathcal{T}_1}\Big(
            \Var_{\tau\sim\pi_b}\big[ {\textstyle \frac{\pi_e(a_{t}\given s_{t})}{\pi_b(a_{t}\given s_{t})}}\given s_1, \ldots, s_T\big] + \E_{\tau\sim\pi_b}\big[ {\textstyle \frac{\pi_e(a_{t}\given s_{t})}{\pi_b(a_{t}\given s_{t})}}\given s_1, \ldots, s_T\big]^2
        \Big) - \prod_{t \in \mathcal{T}_1}\E_{\tau\sim\pi_b}\big[ {\textstyle \frac{\pi_e(a_{t}\given s_{t})}{\pi_b(a_{t}\given s_{t})}}\given s_1, \ldots, s_T\big]^2
        \\
        <
        \prod_{t \in \mathcal{T}_2}\Big(
            \Var_{\tau\sim\pi_b}\big[ {\textstyle \frac{\pi_e(a_{t}\given s_{t})}{\pi_b(a_{t}\given s_{t})}}\given s_1, \ldots, s_T\big] + \E_{\tau\sim\pi_b}\big[ {\textstyle \frac{\pi_e(a_{t}\given s_{t})}{\pi_b(a_{t}\given s_{t})}}\given s_1, \ldots, s_T\big]^2
        \Big) - \prod_{t \in \mathcal{T}_2}\E_{\tau\sim\pi_b}\big[ {\textstyle \frac{\pi_e(a_{t}\given s_{t})}{\pi_b(a_{t}\given s_{t})}}\given s_1, \ldots, s_T\big]^2
    \end{multline}
    We apply Equation~\ref{lem:lr-expectation}:
    \begin{equation}
        \prod_{t \in \mathcal{T}_1}\Big(
            \Var_{\tau\sim\pi_b}\big[ {\textstyle \frac{\pi_e(a_{t}\given s_{t})}{\pi_b(a_{t}\given s_{t})}}\given s_1, \ldots, s_T\big] + 1
        \Big) - 1
        <
        \prod_{t \in \mathcal{T}_2}\Big(
            \Var_{\tau\sim\pi_b}\big[ {\textstyle \frac{\pi_e(a_{t}\given s_{t})}{\pi_b(a_{t}\given s_{t})}}\given s_1, \ldots, s_T\big] + 1
        \Big) - 1
    \end{equation}
    In general, variance is greater than or equal to zero; here the variance is strictly greater than zero because $\frac{\pi_e(a_t \given s_t)}{\pi_b(a_t \given s_t)}$is not constant because we assume $\pi_e\neq\pi_b$. Altogether, because LHS is the product of fewer variances $>0$, we have that LHS $<$ RHS.
\end{proof}

\subsection{Relationship between Trajectory Length and IS Weight}
\label{appendix-sec:pf-jensen}

\begin{proposition}
    For any subsets $\mathcal{T}_1\subsetneq\mathcal{T}_2\subseteq\{1, \ldots, T\}$, if $\pi_e \neq \pi_b$, then
    \begin{equation}
        \displaystyle
        \E_{\tau\sim\pi_b}\big[ \log \prod_{t\in\mathcal{T}_1} \rho_t(\tau) \big]
        >
        \displaystyle
        \E_{\tau\sim\pi_b}\big[ \log \prod_{t\in\mathcal{T}_2} \rho_t(\tau) \big]
    \end{equation}
\end{proposition}

\begin{proof}
    Using the identity for the log of a product and linearity of expectation, the inequality is equivalent to
    \begin{equation}
        \displaystyle
        \sum_{t\in\mathcal{T}_1} \E_{\tau\sim\pi_b}\big[ \log \textstyle \frac{\pi_e(a_t\given s_t)}{\pi_b(a_t\given s_t)} \big]
        >
        \displaystyle
        \sum_{t\in\mathcal{T}_2} \E_{\tau\sim\pi_b}\big[ \log \textstyle \frac{\pi_e(a_t\given s_t)}{\pi_b(a_t\given s_t)} \big]
    \end{equation}

    By Jensen's inequality, for each individual expectation in the sum, $\displaystyle\E_{\tau\sim\pi_b}\big[ \log \textstyle \frac{\pi_e(a_t\given s_t)}{\pi_b(a_t\given s_t)} \big] \leq \log \displaystyle\E_{\tau\sim\pi_b}\big[ \textstyle \frac{\pi_e(a_t\given s_t)}{\pi_b(a_t\given s_t)} \big] = 0$ where we also use Equation~\ref{lem:lr-expectation} ($\displaystyle\E_{\tau\sim\pi_b}\big[ \textstyle \frac{\pi_e(a_t\given s_t)}{\pi_b(a_t\given s_t)} \big] = 1$). Equality holds only if $\frac{\pi_e(a\given s)}{\pi_b(a\given s)}$ is constant, which is not true because we assume $\pi_e\neq \pi_b$. Thus, because the LHS is the sum of fewer expectations $<0$, we have the strict inequality that LHS $>$ RHS.

\end{proof}

\section{Derivations for the General OSIRIS Estimator}

\subsection{Variance}
\label{appendix-sec:OSIRIS-variance-pf}

\begin{theorem}
    Given any mapping $\theta' : \mathcal{S} \to \{0,1\}$, the variance of the OSIRIS estimator is:
    \begin{subequations}
        \begin{gather}
            \Var_{\D\sim\pi_b}[\Vhat_{\OSIRIS}^{\pi_e}(\D; \theta')]
        	=
            \Var_{\D\sim\pi_b}[\Vhat_{\IS}^{\pi_e}(\D)]
            \hspace{1in}\nonumber
        	\\
            + \frac{1}{|\D|}
            \E_{\D\sim\pi_b}[\Vhat_{\IS}^{\pi_e}(\tau^{(1)})]^2
        	- \frac{1}{|\D|}
        	\E_{\D\sim\pi_b}[\Vhat_{\OSIRIS}^{\pi_e}(\tau^{(1)}; \theta')]^2
        	\\
        	- \frac{1}{|\D|}
            	\E_{\D\sim\pi_b}[
            		\Vhat_{\OSIRIS}^{\pi_e}(\tau^{(1)}; \theta')^2
            		]
            	\,
            	\Var_{\D\sim\pi_b}[
            	    \rho_{\theta'}^\complement(\tau^{(1)})
            	    ]
        	\\
        	\hspace{0.25in}
        	- \frac{1}{|\D|}
        	    \Cov_{\D\sim\pi_b}[ \Vhat_{\OSIRIS}^{\pi_e}(\tau^{(1)}; \theta')^2
        	    ,\, \rho_{\theta'}^\complement(\tau^{(1)})^2 ]
        \end{gather}
    \end{subequations}
\end{theorem}

\begin{proof}
    \begin{align}
        \Var_{\D\sim\pi_b}&[\Vhat_\OSIRIS^{\pi_e}(\D)]
        -
        \Var_{\D\sim\pi_b}[\Vhat_\IS^{\pi_e}(\D)]
        \nonumber
        \\
        &=
            \E_{\D\sim\pi_b}[\Vhat_\OSIRIS^{\pi_e}(\D)^2]
            -
            \E_{\D\sim\pi_b}[\Vhat_\IS^{\pi_e}(\D)^2]
            -
            \E_{\D\sim\pi_b}[\Vhat_\OSIRIS^{\pi_e}(\D)]^2
            +
            \E_{\D\sim\pi_b}[\Vhat_\IS^{\pi_e}(\D)]^2
            \label{appendix-step:variance-var-def}
        \\
        &=
            \frac{1}{|\D|}
            \E_{\D\sim\pi_b}[\Vhat_\OSIRIS^{\pi_e}(\tau^{(1)})^2]
            -
            \frac{1}{|\D|}
            \E_{\D\sim\pi_b}[\Vhat_\IS^{\pi_e}(\tau^{(1)})^2]
            -
            \frac{1}{|\D|}
            \E_{\D\sim\pi_b}[\Vhat_\OSIRIS^{\pi_e}(\tau^{(1)})]^2
            +
            \frac{1}{|\D|}
            \E_{\D\sim\pi_b}[\Vhat_\IS^{\pi_e}(\tau^{(1)})]^2
            \label{appendix-step:variance-linear}
        \\
        &=
            \frac{1}{|\D|}
            \E_{\D\sim\pi_b}[\Vhat_\OSIRIS^{\pi_e}(\tau^{(1)})^2]
            -
            \frac{1}{|\D|}
            \E_{\D\sim\pi_b}[
                \Vhat_\OSIRIS^{\pi_e}(\tau^{(1)})^2
                \rho_{\theta'}^\complement(\tau^{(1)})^2
                ]
            \nonumber\\&\qquad
            -
            \frac{1}{|\D|}
            \E_{\D\sim\pi_b}[\Vhat_\OSIRIS^{\pi_e}(\tau^{(1)})]^2
            +
            \frac{1}{|\D|}
            \E_{\D\sim\pi_b}[\Vhat_\IS^{\pi_e}(\tau^{(1)})]^2
            \label{appendix-step:variance-is-split}
        \\
        &=
            \frac{1}{|\D|}
            \E_{\D\sim\pi_b}[\Vhat_\OSIRIS^{\pi_e}(\tau^{(1)})^2]
            -
            \frac{1}{|\D|}
            \E_{\D\sim\pi_b}[
                \Vhat_\OSIRIS^{\pi_e}(\tau^{(1)})^2
                ]
            \E_{\D\sim\pi_b}[
                \rho_{\theta'}^\complement(\tau^{(1)})^2
                ]
            \nonumber\\&\qquad
            -
            \frac{1}{|\D|}
            \E_{\D\sim\pi_b}[\Vhat_\OSIRIS^{\pi_e}(\tau^{(1)})]^2
            +
            \frac{1}{|\D|}
            \E_{\D\sim\pi_b}[\Vhat_\IS^{\pi_e}(\tau^{(1)})]^2
            -
            \frac{1}{|\D|}
            \Cov_{\D\sim\pi_b}[
                \Vhat_\OSIRIS^{\pi_e}(\tau^{(1)})^2
                ,\,
                \rho_{\theta'}^\complement(\tau^{(1)})^2
                ]
            \label{appendix-step:variance-covariance}
        \\
        &=
            \frac{1}{|\D|}
            \E_{\D\sim\pi_b}[\Vhat_\OSIRIS^{\pi_e}(\tau^{(1)})^2]
            \E_{\D\sim\pi_b}[
                \rho_{\theta'}^\complement(\tau^{(1)})
                ]^2
            -
            \frac{1}{|\D|}
            \E_{\D\sim\pi_b}[
                \Vhat_\OSIRIS^{\pi_e}(\tau^{(1)})^2
                ]
            \E_{\D\sim\pi_b}[
                \rho_{\theta'}^\complement(\tau^{(1)})^2
                ]
            \nonumber\\&\qquad
            -
            \frac{1}{|\D|}
            \E_{\D\sim\pi_b}[\Vhat_\OSIRIS^{\pi_e}(\tau^{(1)})]^2
            +
            \frac{1}{|\D|}
            \E_{\D\sim\pi_b}[\Vhat_\IS^{\pi_e}(\tau^{(1)})]^2
            -
            \frac{1}{|\D|}
            \Cov_{\D\sim\pi_b}[
                \Vhat_\OSIRIS^{\pi_e}(\tau^{(1)})^2
                ,\,
                \rho_{\theta'}^\complement(\tau^{(1)})^2
                ]
            \label{appendix-step:variance-lr-expect}
        \\
        &=
            \frac{1}{|\D|}
            \E_{\D\sim\pi_b}[\Vhat_\OSIRIS^{\pi_e}(\tau^{(1)})^2]
            \Big(
                \E_{\D\sim\pi_b}[
                    \rho_{\theta'}^\complement(\tau^{(1)})
                    ]^2
                -
                \E_{\D\sim\pi_b}[
                    \rho_{\theta'}^\complement(\tau^{(1)})^2
                    ]
            \Big)
            \nonumber\\&\qquad
            -
            \frac{1}{|\D|}
            \E_{\D\sim\pi_b}[\Vhat_\OSIRIS^{\pi_e}(\tau^{(1)})]^2
            +
            \frac{1}{|\D|}
            \E_{\D\sim\pi_b}[\Vhat_\IS^{\pi_e}(\tau^{(1)})]^2
            -
            \frac{1}{|\D|}
            \Cov_{\D\sim\pi_b}[
                \Vhat_\OSIRIS^{\pi_e}(\tau^{(1)})^2
                ,\,
                \rho_{\theta'}^\complement(\tau^{(1)})^2
                ]
            \label{appendix-step:variance-factor}
        \\
        &=
            -
            \frac{1}{|\D|}
            \E_{\D\sim\pi_b}[\Vhat_\OSIRIS^{\pi_e}(\tau^{(1)})^2]
            \Var_{\D\sim\pi_b}[
                \rho_{\theta'}^\complement(\tau^{(1)})
                ]
            \nonumber\\&\qquad
            -
            \frac{1}{|\D|}
            \E_{\D\sim\pi_b}[\Vhat_\OSIRIS^{\pi_e}(\tau^{(1)})]^2
            +
            \frac{1}{|\D|}
            \E_{\D\sim\pi_b}[\Vhat_\IS^{\pi_e}(\tau^{(1)})]^2
            -
            \frac{1}{|\D|}
            \Cov_{\D\sim\pi_b}[
                \Vhat_\OSIRIS^{\pi_e}(\tau^{(1)})^2
                ,\,
                \rho_{\theta'}^\complement(\tau^{(1)})^2
                ]
            \label{appendix-step:variance-var-def-again}
    \end{align}
    where we use
    the definition of variance
    (\ref{appendix-step:variance-var-def}),
    linearity of expectation
    (\ref{appendix-step:variance-linear}),
    the relationship between IS and OSIRIS
    (\ref{appendix-step:variance-is-split}),
    the definition of covariance
    (\ref{appendix-step:variance-covariance}),
    Equation~\ref{lem:lr-expectation}
    (\ref{appendix-step:variance-lr-expect}),
    factor terms
    (\ref{appendix-step:variance-factor}), and
    the definition of variance again (\ref{appendix-step:variance-var-def-again}).
\end{proof}

\subsection{Bias}
\label{appendix-sec:OSIRIS-mean-pf}

\begin{theorem}
    Given any mapping $\theta' : \mathcal{S} \to \{0,1\}$, the mean of the OSIRIS estimator is:
    \begin{equation}
        \E_{\D\sim\pi_b}[\Vhat_\OSIRIS^{\pi_e}(\D; \theta')] = 
            V^{\pi_e}
            -
            \Cov_{\D\sim\pi_b}[ \Vhat_\OSIRIS^{\pi_e}(\tau^{(1)}; \theta'), \, \rho_{\theta'}^\complement(\tau^{(1)})]
    \end{equation}
\end{theorem}

\begin{proof}
    From the fact that the IS estimator is unbiased:
    \begin{align}
        V^{\pi_e} &= \E_{\D\sim\pi_b}[\Vhat_\IS(\D)] = \E_{\D\sim\pi_b}[\Vhat_\IS^{\pi_e}(\tau^{(1)})]
        \label{appendix-step:ideal-case-linear}
        \\
        &= \E_{\D\sim\pi_b}[\Vhat_\OSIRIS^{\pi_e}(\tau^{(1)};\theta')\cdot \rho_{\theta'}^\complement(\tau^{(1)})]
        \label{appendix-step:ideal-case-split}
        \\
        &= \E_{\D\sim\pi_b}[\Vhat_\OSIRIS^{\pi_e}(\tau^{(1)};\theta')] \E_{\D\sim\pi_b}[\rho_{\theta'}^\complement(\tau^{(1)})]
            + \Cov_{\D\sim\pi_b}[\Vhat_\OSIRIS^{\pi_e}(\tau^{(1)};\theta'),\, \rho_{\theta'}^\complement(\tau^{(1)})]
        \label{appendix-step:ideal-case-covar}
        \\
        &= \E_{\D\sim\pi_b}[\Vhat_\OSIRIS^{\pi_e}(\tau^{(1)};\theta')]
            + \Cov_{\D\sim\pi_b}[\Vhat_\OSIRIS^{\pi_e}(\tau^{(1)};\theta'),\, \rho_{\theta'}^\complement(\tau^{(1)})]
        \label{appendix-step:ideal-case-expect}
    \end{align}
    where we use the assumption that trajectories are sampled i.i.d. (\ref{appendix-step:ideal-case-linear}), the relationship between IS and OSIRIS (\ref{appendix-step:ideal-case-split}), the definition of covariance (\ref{appendix-step:ideal-case-covar}), and then Equation~\ref{lem:lr-expectation} (\ref{appendix-step:ideal-case-expect}).
\end{proof}

\subsection{Decomposition of the Covariance Term}
\label{appendix-sec:cov-indivlrs}

We can decompose the covariance involving the \textit{product} of likelihood ratios into a sum of covariances involving \textit{individual} likelihood ratios. We use the notation $\rho_{1:\ell}^\complement$ to indicate the product of the 1st to $\ell$th irrelevant likelihood ratios.
\begin{align}
    \Cov_{\D \sim \pi_b}&[\Vhat_\OSIRIS(\tau^{(1)}; \theta'),\, \rho_\OSIRIS^\complement(\tau^{(1)}; \theta')] \nonumber
    \\
    &=
    \E_{\D \sim \pi_b}[
        \Vhat_\OSIRIS(\tau^{(1)}; \theta')
        \rho_\OSIRIS^\complement(\tau^{(1)}; \theta')
        ]
    -
    \E_{\D \sim \pi_b}[
        \Vhat_\OSIRIS(\tau^{(1)}; \theta')
        ]
    \E_{\D \sim \pi_b}[
        \rho_\OSIRIS^\complement(\tau^{(1)}; \theta')
        ]
    \label{appendix-step:covindivlrs-def}
    \\
    &=
    \E_{\D \sim \pi_b}[
        \Vhat_\OSIRIS(\tau^{(1)}; \theta')
        \rho_\OSIRIS^\complement(\tau^{(1)}; \theta')
        ]
    -
    \E_{\D \sim \pi_b}[
        \Vhat_\OSIRIS(\tau^{(1)}; \theta')
        ]
    \label{appendix-step:covindivlrs-lem}
    \\
    &=
    \E_{\D \sim \pi_b}[
        \Vhat_\OSIRIS(\tau^{(1)}; \theta')
        \rho_{1:\ell-1}^\complement(\tau^{(1)}; \theta')
        ]
    \E_{\D \sim \pi_b}[
        \rho_{\ell}^\complement(\tau^{(1)}; \theta')
        ]
    -
    \E_{\D \sim \pi_b}[
        \Vhat_\OSIRIS(\tau^{(1)}; \theta')
        ]
    \nonumber\\&\qquad
    +
    \Cov_{\D \sim \pi_b}[
        \Vhat_\OSIRIS(\tau^{(1)}; \theta')
        \rho_{1:\ell-1}^\complement(\tau^{(1)}; \theta')
        ,\,
        \rho_{\ell}^\complement(\tau^{(1)}; \theta')
        ]
    \label{appendix-step:covindivlrs-covdef-again}
    \\
    &=
    \E_{\D \sim \pi_b}[
        \Vhat_\OSIRIS(\tau^{(1)}; \theta')
        \rho_{1:\ell-1}^\complement(\tau^{(1)}; \theta')
        ]
    -
    \E_{\D \sim \pi_b}[
        \Vhat_\OSIRIS(\tau^{(1)}; \theta')
        ]
    \nonumber\\&\qquad
    +
    \Cov_{\D \sim \pi_b}[
        \Vhat_\OSIRIS(\tau^{(1)}; \theta')
        \rho_{1:\ell-1}^\complement(\tau^{(1)}; \theta')
        ,\,
        \rho_{\ell}^\complement(\tau^{(1)}; \theta')
        ]
    \label{appendix-step:covindivlrs-lem-again}
    \\
    &=
    \E_{\D \sim \pi_b}[
        \Vhat_\OSIRIS(\tau^{(1)}; \theta')
        ]
    -
    \E_{\D \sim \pi_b}[
        \Vhat_\OSIRIS(\tau^{(1)}; \theta')
        ]
    \nonumber\\&\qquad
    +
    \Cov_{\D \sim \pi_b}[
        \Vhat_\OSIRIS(\tau^{(1)}; \theta')
        \rho_{1:\ell-1}^\complement(\tau^{(1)}; \theta')
        ,\,
        \rho_{\ell}^\complement(\tau^{(1)}; \theta')
        ]
    \nonumber\\&\qquad
    +
    \Cov_{\D \sim \pi_b}[
        \Vhat_\OSIRIS(\tau^{(1)}; \theta')
        \rho_{1:\ell-2}^\complement(\tau^{(1)}; \theta')
        ,\,
        \rho_{\ell-1}^\complement(\tau^{(1)}; \theta')
        ]
    \nonumber\\&\qquad
    +\cdots
    \label{appendix-step:covindivlrs-recurse}
\end{align}
where we use the definition of covariance (\ref{appendix-step:covindivlrs-def}), Equation~\ref{lem:lr-expectation} (\ref{appendix-step:covindivlrs-lem}), the definition of covariance again (\ref{appendix-step:covindivlrs-covdef-again}), Equation~\ref{lem:lr-expectation} again (\ref{appendix-step:covindivlrs-lem-again}), and repeat (\ref{appendix-step:covindivlrs-recurse}). Eventually, the $\displaystyle\E_{\D \sim \pi_b}[\Vhat_\OSIRIS(\tau^{(1)}; \theta')]$ will cancel, leaving us with the sum of individual covariances.

\section{Derivations for the OSIRIS Estimator using State Relevance}

\subsection{Composite Policy $\pi_e'$}
\label{appendix-sec:pf-composite-policy-lem}

\begin{lemma} \label{appendix-lem:composite-policy}
    Let $\pi_e'$ be a composite policy:
    \begin{equation}
        \pi_e'(a\given s; \theta) \equiv \begin{cases}
            \pi_e(a\given s) & \text{if } \theta(s)=1 \text{ (relevant)} \\
            \pi_b(a\given s) & \text{if } \theta(s)=0\text{ (irrelevant)}
        \end{cases}.
    \end{equation}
    Then the policy values of $\pi_e$ and $\pi_e'$ are equal:
    \begin{equation}
        V^{\pi_e}=V^{\pi_e'}.
    \end{equation}
\end{lemma}

\begin{proof}
    First, we will recall some notation. The state-action value function under policy $\pi_e$ is
    \begin{equation}
        Q^{\pi_e}(s, a)
        \equiv
        \E_{\tau\sim\pi_e}[g_{t:T}(\tau)\given s_{t}=s,\, a_{t}=a],
    \end{equation}
    and the state value function under policy $\pi_e$ is then
    \begin{equation}
        V^{\pi_e}(s)
        \equiv
        \E_{a\sim\pi_e(\cdot\given s)} \big[
            Q^{\pi_e}(s, a)
        \big].
    \end{equation}
    
    Now, for any state $s$,
    \begin{equation} \label{appendix-step:unbiased-composite-policy-swap}
        V^{\pi_e}(s)
        =
        \E_{a\sim\pi_e(\cdot\given s)} \big[
            Q^{\pi_e}(s, a)
        \big]
        =
        \E_{a\sim\pi_e'(\cdot\given s)} \big[
            Q^{\pi_e}(s, a)
        \big]
    \end{equation}
    because if $\theta(s)=0$, then we can replace $\pi_e$ in the expectation with anything because, by definition of state irrelevance, $Q^{\pi_e}(s,a)$ is constant with respect to $a$. Otherwise, if $\theta(s)=1$, then $\pi_e'$ is simply equivalent to $\pi_e$ by definition of the composite policy. By applying this logic with the recursive Bellman equation $V^{\pi_e}(s) = \E_{a\sim\pi_e(\cdot\given s)} \Big[
            \E_{s'\sim P(\cdot\given s, a)}\big[
                    R(s, a)
                    +
                    \gamma
                    V^{\pi_e}(s')
            \big]
        \Big]$, we conclude that for any state $s$, 
    \begin{equation}
        V^{\pi_e}(s) = V^{\pi_e'}(s).
    \end{equation}
    We take an expectation over the initial state distribution,
    \begin{equation}
        \E_{s_1\sim P(s_1)}\big[ V^{\pi_e}(s_1) \big] = \E_{s_1\sim P(s_1)}\big[ V^{\pi_e'}(s_1) \big],
    \end{equation}
    which is equivalent to
    \begin{equation}
        V^{\pi_e} = V^{\pi_e'}.
    \end{equation}
\end{proof}

\subsection{Bias using $\theta$}
\label{appendix-sec:unbiased-ideal}

\begin{theorem}
    \label{appendix-thm:OSIRIS-bias-ideal}
    Given the true relevance mapping $\theta$, the mean of the OSIRIS estimator is
    \begin{equation}
        \E_{\D\sim\pi_b}[\Vhat_{\OSIRIS}^{\pi_e}(\D; \theta)] = V^{\pi_e}
    \end{equation}
\end{theorem}

\begin{proof}
    The key idea here is that the likelihood ratio omission procedure in OSIRIS is equivalent to pretending $\pi_e'$ is the evaluation policy.
    \begin{align}
        \E_{\D\sim\pi_b}&\big[\Vhat_\OSIRIS^{\pi_e}(\D; \theta)\big] \\
        &= \label{appendix-step:osirismean-iid}
        \E_{\tau\sim\pi_b}\big[\Vhat_\OSIRIS^{\pi_e}(\tau; \theta)\big] \\
        &= \label{appendix-step:osirismean-def}
        \E_{\tau\sim\pi_b}\Big[
            g(\tau)
            \prod_{t=1}^{T} \Big[\frac{\pi_e(a_t\given s_t)}{\pi_b(a_t\given s_t)}\Big]^{\theta(s_{t})}
        \Big] \\
        &= \label{appendix-step:osirismean-integral}
        \int
            g(\tau)
            \prod_{t=1}^{T} \Big[\frac{\pi_e(a_t\given s_t)}{\pi_b(a_t\given s_t)}\Big]^{\theta(s_{t})}
            \pi_b(a_t\given s_t) P(s_{t+1}\given s_{t}, a_{t}) P(s_1)
            \, d\tau
            \\
        &= \label{appendix-step:osirismean-cancel}
        \int
            g(\tau)
            \prod_{t=1}^{T} \big[\pi_e(a_t\given s_t)\big]^{\theta(s_{t})}
            \big[\pi_b(a_t\given s_t)\big]^{1-\theta(s_{t})} P(s_{t+1}\given s_{t}, a_{t})
            P(s_1)
            \, d\tau
            \\
        &= \label{appendix-step:osirismean-composite-policy}
        \int
            g(\tau)
            \prod_{t=1}^{T} \pi_e'(a_t\given s_t) P(s_{t+1}\given s_{t}, a_{t})
            P(s_1)
            \, d\tau
            \\
        &= \label{appendix-step:osirismean-finish}
        \E_{\tau\sim\pi_e'}[
            g(\tau)
        ] =
        V^{\pi_e'} =
        V^{\pi_e}
    \end{align}
    where we use the assumption that trajectories are sampled i.i.d. (\ref{appendix-step:osirismean-iid}), the definition of the OSIRIS estimator (\ref{appendix-step:osirismean-def}), the definition of the expectation (\ref{appendix-step:osirismean-integral}), cancellation of terms (\ref{appendix-step:osirismean-cancel}), and the definition of the composite policy $\pi_e'$ (\ref{appendix-step:osirismean-composite-policy}). Finally, we clean up by reintroducing the expectation, using Lemma~\ref{appendix-lem:composite-policy}, and recovering the definition of the true policy value (\ref{appendix-step:osirismean-finish}).
\end{proof}

Note that if state $s$ is truly relevant but was omitted by OSIRIS, then bias would be introduced from Equation \ref{appendix-step:unbiased-composite-policy-swap}:
\begin{equation}
    \Big| \E_{a\sim\pi_b} [Q^{\pi_e}(s, a)]
    -
    \E_{a\sim\pi_e} [Q^{\pi_e}(s, a)]
    \Big| \neq 0
\end{equation}

However, if state $s$ is truly irrelevant but was kept by OSIRIS, then there would be no effect on bias because Equation \ref{appendix-step:unbiased-composite-policy-swap} still holds because $\pi_e$ in the expectation can be replaced with anything.

\subsection{Consistency using $\thetahat$}
\label{appendix-sec:pf-consistent}

\begin{theorem}
    If $|\mathcal{A}|=2$ and $\alpha > 0$, then as $|\D|\to\infty$
    \begin{equation}
        \E_{\D\sim\pi_b}[\Vhat_\OSIRIS^{\pi_e}(\D; \thetahat(\,\cdot\,; \D))] = V^{\pi_e}
    \end{equation}
\end{theorem}

\begin{proof}
    Welch's two-sample $t$-test is consistent, which means that for all $s\in\mathcal{S}$ where $\theta(s)=1$, the test will give $\thetahat(s)=1$ as $|\D|\to\infty$. The binary classification of actions when calculating $\hat{\theta}$ has no effect on this fact because the action space is already assumed to be binary. Thus, as $|\D|\to\infty$, all relevant likelihood ratios will be kept by OSIRIS, so in this limit, the estimator will be unbiased (Appendix~\ref{appendix-sec:unbiased-ideal}).
\end{proof}

\subsection{Step-Wise OSIRIS}
\label{appendix-sec:stepwise-OSIRIS}

\begin{theorem}
    Let state $s\in\mathcal{S}$ be irrelevant to the reward $\Delta t$-steps away if
    \begin{equation}
        \E_{\tau\sim\pi_e}[r_{t+\Delta t}\given s_{t}=s,a_{t}=a] = \mathrm{constant}
        ,\quad
        \forall a\in\mathcal{A}.
    \end{equation}
    Otherwise, $s$ is relevant to the reward $\Delta t$-steps away. Using this condition, we define $\theta_{\Delta t}:\mathcal{S}\to\{0,1\}$ where $\theta_{\Delta t}(s)=0$ if $s$ is irrelevant to the reward $\Delta t$-steps away, and otherwise $\theta_{\Delta t}(s)=1$. Then
    \begin{equation}
        \hat{V}_{\substack{\textnormal{step-wise} \\ \textnormal{OSIRIS}}}^{\pi_e}(\D; \theta_{\Delta t})
        \equiv
        \frac{1}{|\D|}\sum_{\tau\in\D}
        \sum_{t'=1}^{T}
        \Big(
            \gamma^{t'-1}
            r_{t'}
            \prod_{t=1}^{T} \big[\rho_{t}(\tau)\big]^{\theta_{t' - t}(s_{t})}
        \Big)
    \end{equation}
    is an unbiased estimator of $V^{\pi_e}$.
\end{theorem}

\begin{proof}
    This proof is analogous to Sections~\ref{appendix-sec:pf-composite-policy-lem} and \ref{appendix-sec:unbiased-ideal} where we use a composite policy
    \begin{equation} \label{appendix-eqn:stepwise-composite-policy}
        \pi_e'(a\given s; \theta_{\Delta t}) \equiv \begin{cases}
            \pi_e(a\given s) & \text{if } \theta_{\Delta t}(s)=1 \text{ (relevant)} \\
            \pi_b(a\given s) & \text{if } \theta_{\Delta t}(s)=0\text{ (irrelevant)}
        \end{cases}.
    \end{equation}
    We will first prove the lemma in Equation~\ref{appendix-eqn:lemma-stepwise} and next prove that the step-wise OSIRIS estimator is equivalent to pretending $\pi_e'$ is the evaluation policy.
    
    First, the expected reward at each individual time step $t'$ is
    \begin{align}
        \E_{\tau\sim\pi_e}\big[
            R(s_{t'}, a_{t'})
        \big]
        &=
        \E_{s_t\sim P(s_t)}\big[
            \E_{a_t\sim \pi_e(a_t\given s_t)}\big[
                \E_{\tau\sim\pi_e}\big[
                    R(s_{t'}, a_{t'})
                    \given s_t, a_t
                \big]
            \big]
        \big]
        \\
        &=
        \E_{s_t\sim P(s_t)}\big[
            \E_{a_t\sim \pi_e'(a_t\given s_t)}\big[
                \E_{\tau\sim\pi_e}\big[
                    R(s_{t'}, a_{t'})
                    \given s_t, a_t
                \big]
            \big]
        \big]
    \end{align}
    where we use the law of total expectation by introducing $P(s_t)$ as the stationary state distribution. Then we introduce the composite policy defined in Equation~\ref{appendix-eqn:stepwise-composite-policy} because $\pi_e'$ is equivalent to $\pi_e$ except in states that are irrelevant (to the reward $(t' - t)$-steps away) where, by definition, we can replace $\pi_e$ in the expectation with anything. Recursively applying this logic, we find that
    \begin{equation} \label{appendix-eqn:lemma-stepwise}
        \E_{\tau\sim\pi_e}\big[
            R(s_{t'}, a_{t'})
        \big]
        =
        \E_{\tau\sim\pi_e'}\big[
            R(s_{t'}, a_{t'})
        \big]
    \end{equation}
    Now, we will use this fact to show the step-wise OSIRIS estimator is unbiased:
    \begin{align}
        \E_{\D\sim\pi_b}&\big[\Vhat_{\substack{\textnormal{step-wise} \\ \textnormal{OSIRIS}}}^{\pi_e}(\D; \theta_{\Delta t})\big] \\
        &= \label{appendix-step:stepwise-iid}
        \E_{\tau\sim\pi_b}\big[\Vhat_{\substack{\textnormal{step-wise} \\ \textnormal{OSIRIS}}}^{\pi_e}(\tau; \theta_{\Delta t})\big] \\
        &= \label{appendix-step:stepwise-def}
        \E_{\tau\sim\pi_b}\Big[
            \sum_{t'=1}^{T}
            \gamma^{t'-1}
            r_{t'}
            \prod_{t=1}^{T} \Big[\frac{\pi_e(a_t\given s_t)}{\pi_b(a_t\given s_t)}\Big]^{\theta_{t' - t}(s_{t})}
        \Big] \\
        &= \label{appendix-step:stepwise-linearity}
        \sum_{t'=1}^{T}
        \gamma^{t'-1}
        \E_{\tau\sim\pi_b}\Big[
            r_{t'}
            \prod_{t=1}^{T} \Big[\frac{\pi_e(a_t\given s_t)}{\pi_b(a_t\given s_t)}\Big]^{\theta_{t' - t}(s_{t})}
        \Big] \\
        &= \label{appendix-step:stepwise-integral}
        \sum_{t'=1}^{T}
        \gamma^{t'-1}
        \int
            r_{t'}
            \prod_{t=1}^{T} \Big[\frac{\pi_e(a_t\given s_t)}{\pi_b(a_t\given s_t)}\Big]^{\theta_{t' - t}(s_{t})}
            \pi_b(a_t\given s_t) P(s_{t+1}\given s_{t}, a_{t}) P(s_1)
            \, d\tau
            \\
        &= \label{appendix-step:stepwise-cancel}
        \sum_{t'=1}^{T}
        \gamma^{t'-1}
        \int
            r_{t'}
            \prod_{t=1}^{T} \big[\pi_e(a_t\given s_t)\big]^{\theta_{t' - t}(s_{t})}
            \big[\pi_b(a_t\given s_t)\big]^{1-\theta_{t' - t}(s_{t})} P(s_{t+1}\given s_{t}, a_{t})
            P(s_1)
            \, d\tau
            \\
        &= \label{appendix-step:stepwise-composite-policy}
        \sum_{t'=1}^{T}
        \gamma^{t'-1}
        \int
            r_{t'}
            \prod_{t=1}^{T} \pi_e'(a_t\given s_t) P(s_{t+1}\given s_{t}, a_{t})
            P(s_1)
            \, d\tau
            \\
        &= \label{appendix-step:stepwise-finish}
        \sum_{t'=1}^{T}
        \gamma^{t'-1}
        \E_{\tau\sim\pi_e'}[
            r_{t'}
        ] =
        \sum_{t'=1}^{T}
        \gamma^{t'-1}
        \E_{\tau\sim\pi_e}[
            r_{t'}
        ] =
        V^{\pi_e}
    \end{align}
    where we use the assumption that trajectories are sampled i.i.d. (\ref{appendix-step:stepwise-iid}), the definition of the step-wise OSIRIS estimator (\ref{appendix-step:stepwise-def}), linearity of expectation (\ref{appendix-step:stepwise-linearity}), the definition of the expectation (\ref{appendix-step:stepwise-integral}), cancellation of terms (\ref{appendix-step:stepwise-cancel}), and the definition of the composite policy $\pi_e'$ (\ref{appendix-step:stepwise-composite-policy}). Finally, we clean up by reintroducing the expectation, using Equation~\ref{appendix-eqn:lemma-stepwise}, and recovering the definition of the true policy value (\ref{appendix-step:stepwise-finish}).
\end{proof}

\section{Environment Descriptions}
\label{appendix-sec:env-desc}

All reported results are aggregated over 200 independent trials with discount factor $\gamma=1$.

\subsection{Gridworlds}

\begin{figure}[h]
\vskip 0.2in
\begin{center}
\centerline{\includegraphics{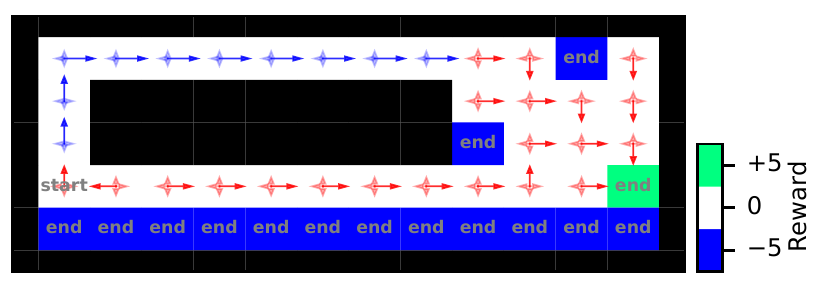}}
\caption{Environment and policies for Gridworld experiments.}
\label{appendix-fig:gridworld-env}
\end{center}
\vskip -0.1in
\end{figure}

The Gridworld environment and policies are shown in Figure~\ref{appendix-fig:gridworld-env}. From each state, the agent takes the action indicated by the large arrow with probability $1 - \epsilon$ and takes a random action with probability $\epsilon$. The evaluation policy has $\epsilon=0.1$, and the behavior policy has $\epsilon=0.5$. Rewards are given for entering the indicated states. The data size is $|\D|=25$. These parameters describe the \textit{Dilly-Dallying Gridworld}. The \textit{Express Gridworld} variant is identical except that the behavior policy has $\epsilon=0.2$ only in the corridor states (blue arrows in Figure~\ref{appendix-fig:gridworld-env}).

\subsection{Lunar Lander}

\textit{Lunar Lander} is a popular deterministic benchmark environment in OpenAI gym with an 8-dimensional continuous state space and 4 discrete actions \citepappendix{brockman2016openai}. The goal is to safely land on the ground by firing engines that push the agent left, right, or up. The reward from each transition is related to the agent's distance from the target landing site with small penalties for firing engines; around 100-140 is accumulated for a typical landing. A successful landing earns an additional $+100$ whereas a crash earns $-100$, and both events end the trajectory. Or the trajectory is automatically ended after 1000 steps.

We used $\epsilon$-greedy behavior and evaluation policies based on an optimal policy learned by DQN \citepappendix{mnih2013dqn}. The behavior policy takes a random action with probability $\epsilon=0.1$, and the evaluation policy does so with probability $\epsilon=0.05$. The data size is $|\D|=50$.

Calculating $\thetahat(s;\D)$ requires a discrete state space, so only for this step of the OSIRIS algorithm, we discretized the state space by creating three linearly spaced bins per state dimension. In each trial $\D$, this procedure resulted in about 200 discrete states that were represented at least once.

\subsection{Cart Pole}

\textit{Cart Pole} is another deterministic benchmark environment in OpenAI gym. It has a 4-dimensional continuous state space and 2 discrete actions \citepappendix{barto83,brockman2016openai}. The goal is to apply leftward or rightward force to the bottom (the cart) of an inverted pendulum (the pole) to keep it balanced upright. The agent gets $+1$ reward for each transition until the the pole falls, which ends the trajectory.

The behavior and evaluation policies and the discrete state space were obtained as described for Lunar Lander, but the policies here used $\epsilon=0.25$ and $\epsilon=0.2$, respectively. The data size is $|\D|=50$. In each trial $\D$, about 16 discrete states were represented at least once.

\section{Extended Results}

\subsection{Alternative Implementations}
\label{appendix-sec:alternative-implementations}

To demonstrate the extent to which our analysis is robust to specific implementation choices, here we present empirical results for other possible implementations (Table~\ref{appendix-tab:estimator-mse}).

The \textit{Algorithm~\ref{alg:estimate-relevance}} implementation refers to the procedure presented in the main text with Welch's two-sample $t$-test. We also tried using \textit{Smirnov}'s non-parametric test \citepappendix{Hodges1958} where we increased the significance level to $\alpha=0.2$.

These two-sample statistical tests require that we binarize the action space and discretize the state space. First, in the main text, we proposed binarizing the actions based on their likelihood ratios. Here, we also tried assigning action classes based on whether they led to above- or below-average trajectory returns, which is the \textit{$g(\tau)$-Binary $\mathcal{A}$} implementation. Finally, we also tried directly regressing $\hat{Q}^{\pi_e}$ with a neural network model. This \textit{NN as $\hat{Q}^{\pi_e}$} implementation can certainly handle continuous states and it can handle action spaces $|\mathcal{A}|>2$. The model had one 32-dimensional hidden layer. For each trial (i.e. sample of $\D$), it was trained to minimize the Huber loss for 500 epochs, each looking at a minibatch of 128 transitions. We use the trained $\hat{Q}^{\pi_e}$ to calculate $\hat{\theta}(s)$ for a given $s$. First, we calculate the standard deviation of $\hat{Q}^{\pi_e}(s, a)$ for all $a\in\mathcal{A}$. We normalize this by dividing by the mean of $\hat{Q}^{\pi_e}(s, a)$ for all $a\in\mathcal{A}$. This value gives us a sense of how much the $Q^{\pi_e}$ value function varies for different actions taken from the same state. If the value is greater than $\alpha = 0.4$, then we output $\hat{\theta}(s)=1$ (relevant), and otherwise, $\hat{\theta}(s)=0$ (irrelevant).

\begin{table*}[t]
\caption{Comparison of mean squared errors for alternative implementations of OSIRIS state relevance implementation.} \label{appendix-tab:estimator-mse}
\vskip 0.15in
\begin{center}
\begin{small}
\begin{sc}
\begin{tabular}{ll|rrrrrrrr|r}
\toprule
&& \multicolumn{2}{b{0.8in}}{Algorithm 1} & \multicolumn{2}{b{0.8in}}{$g_\tau$-Binary $\mathcal{{A}}$} & \multicolumn{2}{b{0.8in}}{Smirnov} & \multicolumn{2}{b{0.8in}}{NN as $\hat{{Q}}^{{\pi_e}}$} & \multirow{2}{0.4in}{On-Policy} \\
&& osiris & osirwis & osiris & osirwis & osiris & osirwis & osiris & osirwis &  \\
\midrule
\multirow{3}{0.8in}{\textbf{Dilly-Dallying Gridworld}}
    & Mean & $\mathbf{1.3}$ & $1.1$ & $0.5$ & $0.4$ & $0.5$ & $0.5$ & $0.9$ & $0.0$ & $4.3$ \\
    & Std  & $2.0$ & $1.8$ & $\mathbf{1.7}$ & $1.8$ & $1.7$ & $1.8$ & $5.3$ & $2.4$ & $0.6$ \\
    & RMSE & $\mathbf{3.6}$ & $3.7$ & $4.2$ & $4.3$ & $4.1$ & $4.2$ & $6.3$ & $4.9$ & $0.6$ \\
\hline
\multirow{3}{0.8in}{\textbf{Express Gridworld}}
    & Mean & $0.7$ & $0.8$ & $0.7$ & $\mathbf{0.9}$ & $0.5$ & $0.6$ & $0.7$ & $0.3$ & $4.3$ \\
    & Std  & $\mathbf{1.2}$ & $1.3$ & $1.3$ & $1.5$ & $1.3$ & $1.4$ & $2.9$ & $2.1$ & $0.6$ \\
    & RMSE & $3.8$ & $\mathbf{3.7}$ & $3.8$ & $3.7$ & $4.0$ & $4.0$ & $4.6$ & $4.5$ & $0.6$ \\
\hline
\multirow{3}{0.8in}{\textbf{Cart Pole}}
    & Mean & $\mathbf{1068.9}$ & $759.7$ & $970.0$ & $622.3$ & $582.0$ & $640.7$ & $608.2$ & $608.2$ & $1503.6$ \\
    & Std  & $3961.8$ & $318.5$ & $5337.1$ & $381.8$ & $619.0$ & $308.3$ & $\mathbf{97.6}$ & $\mathbf{97.6}$ & $244.8$ \\
    & RMSE & $3985.5$ & $\mathbf{809.2}$ & $5363.7$ & $960.4$ & $1110.2$ & $916.3$ & $900.7$ & $900.7$ & $244.8$ \\
\hline
\multirow{3}{0.8in}{\textbf{Lunar Lander}}
    & Mean & $\mathbf{244.6}$ & $234.8$ & $241.2$ & $259.7$ & $286.9$ & $248.9$ & $222.9$ & $249.4$ & $245.3$ \\
    & Std  & $55.3$ & $23.5$ & $230.3$ & $\mathbf{13.5}$ & $171.7$ & $16.3$ & $62.0$ & $14.4$ & $6.8$ \\
    & RMSE & $55.3$ & $25.7$ & $230.3$ & $19.7$ & $176.6$ & $16.7$ & $65.9$ & $\mathbf{15.0}$ & $6.8$ \\
\bottomrule
\end{tabular}
\end{sc}
\end{small}
\end{center}
\end{table*}

In Table~\ref{appendix-tab:estimator-mse-oracle}, we also provide empirical results for the OSIRIS estimator using an ``oracle'' state relevance mapping where $\thetahat(s)=0$ (irrelevant) for all $s$ in the Gridworld corridor, and otherwise $\thetahat(s)=1$ (relevant). These results represent a practical bound on the accuracy improvement from likelihood ratio omission. As such, they also give a rough picture of the amount of variance contributed from estimating $\thetahat$ vs noise inherent to the data or variance contributed by OSIRIS's manipulating trajectory lengths (which is expected to be small).

\begin{table*}[t]
\caption{Mean squared errors for OSIRIS using an ``oracle'' state relevance mapping.} \label{appendix-tab:estimator-mse-oracle}
\vskip 0.15in
\begin{center}
\begin{small}
\begin{sc}
\begin{tabular}{ll|rr}
\toprule
&& \multicolumn{2}{b{0.8in}}{Oracle} \\
&& osiris & osirwis \\
\midrule
\multirow{3}{0.8in}{\textbf{Dilly-Dallying Gridworld}}
    & Mean & $3.1$ & $3.4$ \\
    & Std  & $1.9$ & $1.2$ \\
    & RMSE & $2.3$ & $1.5$ \\
\hline
\multirow{3}{0.8in}{\textbf{Express Gridworld}}
    & Mean & $2.9$ & $2.9$ \\
    & Std  & $3.1$ & $1.8$ \\
    & RMSE & $3.4$ & $2.3$ \\
\bottomrule
\end{tabular}
\end{sc}
\end{small}
\end{center}
\end{table*}

\subsection{Express Gridworld}
\label{appendix-sec:gridworldxp}

See Figures~\ref{appendix-fig:mseest}, \ref{appendix-fig:scatter}, and \ref{appendix-fig:gridworld-relevance-xp} for empirical results in the Express Gridworld environment. This figures show the same qualitative trends as reported in the main text for the Dilly-Dallying Gridworld. Figure \ref{appendix-fig:gridworld-relevance-xp} shows OSIRIS was more likely to label corridor states are relevant in Express Gridworld than Dilly-Dallying Gridworld, which is expected because there are fewer visits to the corridor states, so there is likely more noise being picked up by the statistical test.

\begin{figure*}[t]
\vskip 0.2in
\begin{center}
 \begin{subfigure}[b]{2.12in}
     \includegraphics{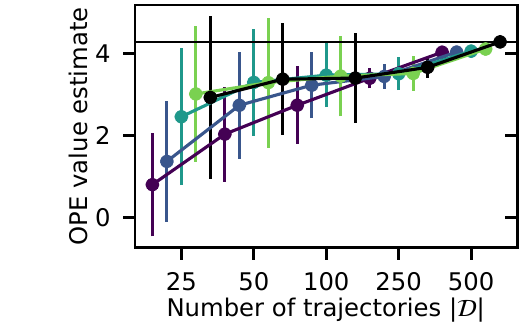}
 \end{subfigure}%
 \begin{subfigure}[b]{0.39in}
     \includegraphics{figures/fig2_legend}
 \end{subfigure}
\caption{Express Gridworld. Distributions of OSIRWIS estimates showing estimator consistency. Dots represent means, error bars represent standard deviations, and horizontal line represents the mean of the on-policy MC estimator. Colors indicate $\alpha$ values, where $\alpha=1$ is equivalent to ordinary WIS.}
\label{appendix-fig:mseest}
\end{center}
\vskip -0.2in
\end{figure*}

\begin{figure*}[t]
\vskip 0.2in
\begin{center}
 \begin{subfigure}[b]{2.12in}
     \includegraphics{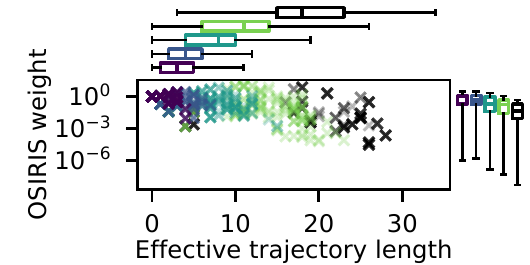}
 \end{subfigure}%
 \begin{subfigure}[b]{0.39in}
     \includegraphics{figures/fig3_legend}
 \end{subfigure}
\caption{Express Gridworld. Scatter plots show correlation between OSIRIS weights and effective trajectory lengths $\sum_{t=1}^{T}\thetahat(s_t)$. Boxplots show variance reduction of OSIRIS weights by shortening and evening of the effective trajectory lengths. Colors indicate $\alpha$ values, where $\alpha=1$ is equivalent to ordinary IS.}
\label{appendix-fig:scatter}
\end{center}
\vskip -0.2in
\end{figure*}

\begin{figure*}[t]
\vskip 0.2in
\begin{center}
\begin{subfigure}[t]{3.25in}
    \includegraphics{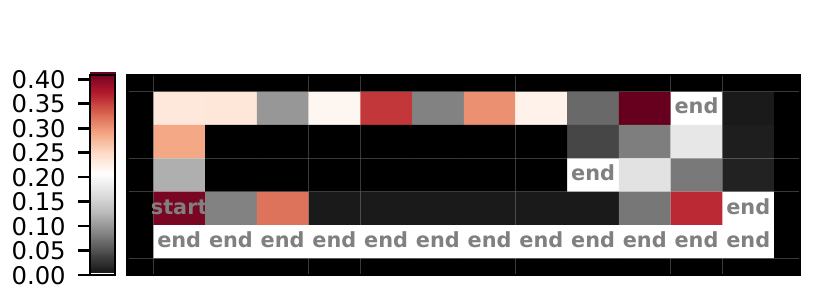}
    \caption{Express Gridworld}
    \label{appendix-fig:gridworld-relevance-xp}
\end{subfigure}
\vskip 0.2in
\begin{subfigure}[t]{\textwidth}
\begin{subfigure}[t]{1.6875in}
    \includegraphics{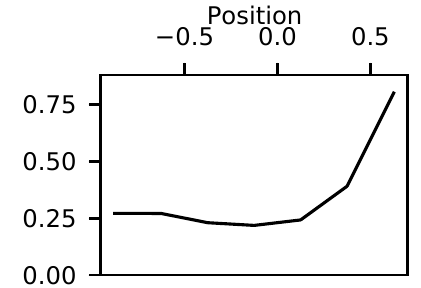}
\end{subfigure}%
\begin{subfigure}[t]{1.6875in}
    \includegraphics{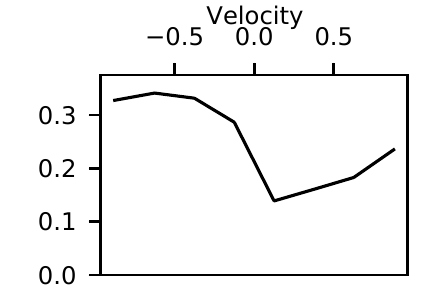}
\end{subfigure}%
\begin{subfigure}[t]{1.6875in}
    \includegraphics{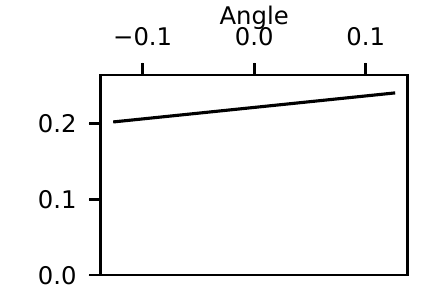}
\end{subfigure}%
\begin{subfigure}[t]{1.6875in}
    \includegraphics{figures/fig4b_relevance_CartPole_dim3}
\end{subfigure}
\caption{Cart Pole}
\label{appendix-fig:cartpole-relevance}
\end{subfigure}
\vskip 0.2in
\begin{subfigure}[t]{\textwidth}
\begin{subfigure}[t]{1.6875in}
    \includegraphics{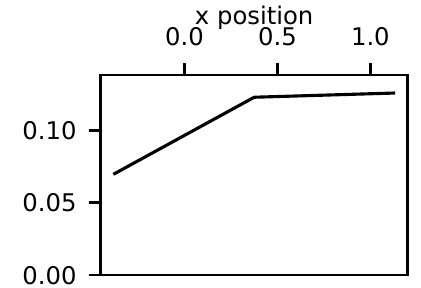}
\end{subfigure}%
\begin{subfigure}[t]{1.6875in}
    \includegraphics{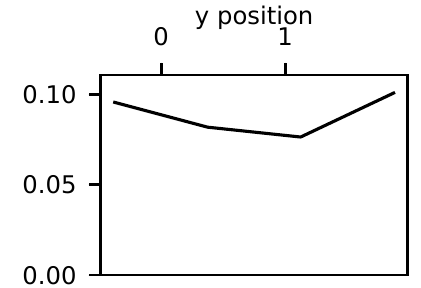}
\end{subfigure}%
\begin{subfigure}[t]{1.6875in}
    \includegraphics{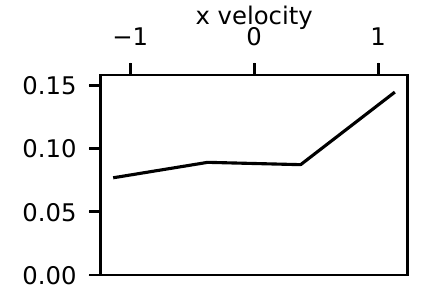}
\end{subfigure}%
\begin{subfigure}[t]{1.6875in}
    \includegraphics{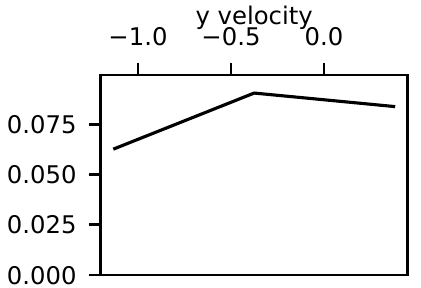}
\end{subfigure}\\
\begin{subfigure}[t]{1.6875in}
    \includegraphics{figures/fig4c_relevance_LunarLander_dim4}
\end{subfigure}%
\begin{subfigure}[t]{1.6875in}
    \includegraphics{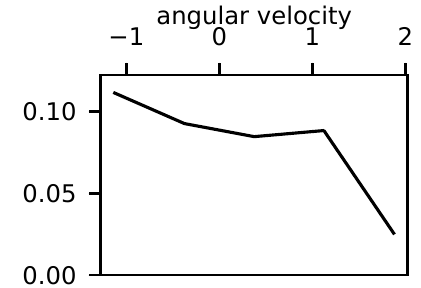}
\end{subfigure}%
\begin{subfigure}[t]{1.6875in}
    \includegraphics{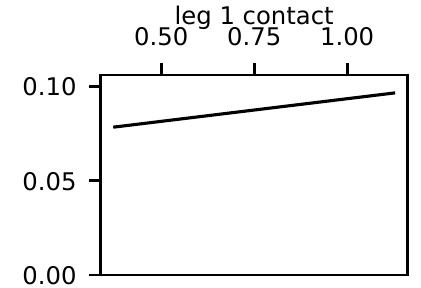}
\end{subfigure}%
\begin{subfigure}[t]{1.6875in}
    \includegraphics{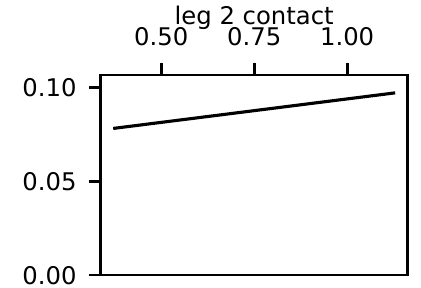}
\end{subfigure}
\caption{Lunar Lander}
\label{appendix-fig:lunarlander-relevance}
\end{subfigure}
\caption{Mean of estimated state relevance $\thetahat(s)$ from visits to the indicated states is represented by color (\subref{appendix-fig:gridworld-relevance-xp}) or on the $y$ axis (\subref{appendix-fig:cartpole-relevance}, \subref{appendix-fig:lunarlander-relevance}). States identified as relevant (i.e. $\thetahat(s)=1$) are key decision points where trajectory outcome is sensitive to action taken.}
\label{appendix-fig:relevance}
\end{center}
\vskip -0.2in
\end{figure*}

\subsection{Estimated State Relevance}
\label{appendix-sec:relevance-other-dims}

We show the estimated state relevance over all state dimensions in Cart Pole (Figure~\ref{appendix-fig:cartpole-relevance}) and Lunar Lander (Figure~\ref{appendix-fig:lunarlander-relevance}).

\bibliographyappendix{bibliography}
\bibliographystyleappendix{icml2021}

\end{document}